\documentclass[letter, 12pt]{article}
\usepackage[margin=1in]{geometry}
\usepackage{blindtext}

\usepackage[preprint]{jmlr2e}
\usepackage{comment, hyperref}
\usepackage{amsmath, xurl}
\usepackage{alltt}
\usepackage{fancyvrb}
\usepackage{subcaption}
\usepackage{bbm}

\newcommand{\M}{\mathcal{M}}
\newcommand{\e}{\mathbf{e}}
\newcommand{\D}{\mathcal{D}}
\newcommand\norm[1]{\left\lVert#1\right\rVert}

\DeclareMathOperator*{\argmin}{arg\,min}

\def\pkg#1{\emph{#1}}
\def\fn#1{\texttt{#1}}
\def\code#1{\texttt{#1}}

\DefineVerbatimEnvironment{example}{Verbatim}{}
\renewenvironment{example*}{\begin{alltt}}{\end{alltt}}

\usepackage{lastpage}

\ShortHeadings{Giddens and Liu}{Giddens and Liu}
\firstpageno{1}

\begin{document}

\title{DPpack: An R Package for Differentially Private Statistical Analysis and Machine Learning}

\author{\name Spencer Giddens \email sgiddens@nd.edu \\
       \addr Department of Applied and Computational Mathematics and Statistics\\
       University of Notre Dame, Notre Dame, IN 46556, USA
       \AND
       \name Fang Liu \email fliu2@nd.edu \\
       \addr Department of Applied and Computational Mathematics and Statistics\\
       University of Notre Dame, Notre Dame, IN 46556, USA}

\editor{My editor}

\maketitle

\vspace{-9pt}\begin{abstract}
Differential privacy (DP) is the state-of-the-art framework for guaranteeing privacy for individuals when releasing aggregated statistics or building statistical/machine learning models from data.
We develop the open-source R package \pkg{DPpack} that provides a large toolkit of differentially private analysis. 
The current version of \pkg{DPpack} implements three popular mechanisms for ensuring DP: Laplace, Gaussian, and exponential.
Beyond that, \pkg{DPpack} provides a large toolkit of easily accessible privacy-preserving descriptive statistics functions.
These include mean, variance, covariance, and quantiles, as well as histograms and contingency tables.
Finally, \pkg{DPpack} provides user-friendly implementation of privacy-preserving versions of logistic regression, SVM, and linear regression, as well as differentially private hyperparameter tuning for each of these models.
This extensive collection of implemented differentially private statistics and models permits hassle-free utilization of differential privacy principles in commonly performed statistical analysis.
We plan to continue developing \pkg{DPpack} and make it more comprehensive by including more differentially private machine learning techniques, statistical modeling and inference in the future.
\end{abstract}


\section{Introduction}\vspace{-3pt}

Data is an invaluable resource harnessed to inform impactful technology development and guide decision-making.
However, utilizing data that contain personally sensitive information (e.g., medical or financial records) poses privacy challenges.
Anonymized datasets, as well as statistics and models derived from sensitive datasets are susceptible to attacks that may result in the leakage of private information \citep{Narayanan2008,Ahn2015,Sweeney2015,Shokri2017,Zhao2021}.
As technology continues to evolve to become more data-reliant, privacy issues will become increasingly more prevalent, necessitating easy access to tools that provide privacy guarantees when releasing information from sensitive datasets.

Differential privacy (DP) \citep{Dwork2006Calibrating} is a popular state-of-the-art framework for providing provable guarantees of privacy for outputs from a statistical or machine learning (ML) procedure.
A variety of randomized procedures and mechanisms exist to achieve DP guarantees for a wide range of analyses. 
These include, to list some examples, summary statistics \citep{Dwork2006Calibrating,Smith2011}, empirical risk minimization \citep{chaudhuri2011,Kifer2012}, classifiers \citep{Chaudhuri2009,Vaidya2013}, deep learning \citep{Abadi2016,Bu2020}, Bayesian networks \citep{Zhang2017}, Bayesian procedures \citep{Dimitrakakis2014, Wang2015}, statistical hypothesis testing \citep{Gaboardi2016,Couch2019,Barrientos2019}, confidence interval construction \citep{Karwa2018,Wang2019}, and synthetic data generation \citep{zhang2017privbayes,torkzadehmahani2019dp,bowen2020comparative}.
Privacy-preserving analysis has also been adopted by many companies in the technology sector, including Google \citep{GoogleDP}, Apple \citep{AppleDP}, Meta \citep{MetaDP}, as well as government agencies like the U.S. Census Bureau \citep{USCensusDP}. 

Given the popularity of DP, many open-source projects have been devoted to developing tools and code for privacy-preserving analysis with DP.
The \pkg{OpenDP Library} \citep{OpenDPpackage}, Google's DP libraries \citep{GoogleDPpackage} (with accompanying OpenMined Python wrapper \pkg{PyDP} \citep{OpenMinedDPpackage}), the \pkg{TensorFlow Privacy} library \citep{TensorFlowPrivacy}, and the IBM's DP library \pkg{diffprivlib} \citep{diffprivlib} collectively provide tools for DP analysis for the Rust, Python, C++, Go, and Java programming languages.

This paper presents the \pkg{DPpack}R package (\url{https://github.com/sgiddens/DPpack}) \citep{DPpack}, which provides convenient implementations of common DP procedures.
R is arguably the most popular languages among statisticians.
Prior to \pkg{DPpack}, R packages \pkg{diffpriv} \citep{diffpriv} and \pkg{PrivateLR} \citep{PrivateLR} were the only available DP R packages.
Both of these packages are limited in scope compared to \pkg{DPpack}; \pkg{PrivateLR}, in fact, implements only a single function for DP logistic regression.
Additionally, neither package has seen an update in the last five years.
Meanwhile, \pkg{DPpack} has been downloaded from CRAN by R users $\sim$4,000 times as of September 2023, averaging 242 downloads per complete month of being available, overtaking \pkg{diffpriv} and \pkg{PrivateLR} to become the most downloaded DP-focused R package in the past 10 months.
While \pkg{diffpriv} does implement several randomized mechanisms for DP, \pkg{DPpack} goes well beyond these basic mechanisms by specifically implementing privacy-preserving versions of various commonly used descriptive statistics, as well as statistical analysis and machine learning procedures.
The implemented functions are accessible even to individuals without a strong background in DP because sensitivity calculations are handled internally based on proven theoretical results and user-provided bounds on the input data.
This makes \pkg{DPpack} more user-friendly than \pkg{diffpriv} for non-expert users.
Even for DP experts, \pkg{DPpack} is attractive due to its scope.
No other R package implements as extensive a collection of privacy-preserving functions.
We plan to continue to develop and update the package by adding more privacy-preserving analysis procedures in the future.

\vspace{-6pt}\section{Capabilities}\label{sec:capabilities}
\vspace{-3pt}\subsection{Randomized Mechanisms}\vspace{-3pt}
\pkg{DPpack} provides the \fn{LaplaceMechanism}, \fn{GaussianMechanism}, and \fn{ExponentialMechanism} functions for implementing general mechanisms for ensuring DP for a desired output.

The \fn{LaplaceMechanism} function implements the Laplace mechanism \citep{Dwork2006Calibrating} for ensuring $\epsilon$-DP for a statistical analysis or function by adding to the output Laplacian noise with a scale parameter dependent on the function's $\ell_1$-global sensitivity and the privacy budget $\epsilon$.
The function generalizes using DP composition to multidimensional function inputs, in which case it allows the user to specify the allocation of the privacy budget across the multiple computations.

The \fn{GaussianMechanism} function implements the Gaussian mechanism \citep{DPtextbook}.
It can be used to ensure either approximate $(\epsilon, \delta)$-DP \citep{Dwork2006OurData}, or probabilistic $(\epsilon, \delta)$-DP \citep{Machanavajjhala2008, Liu2019Generalized}, depending on user input.
It adds Gaussian noise with a variance dependent on $\epsilon$, $\delta$, and the function's $\ell_2$-global sensitivity, and can be generalized to multidimensional inputs.

The \fn{ExponentialMechanism} function implements the exponential mechanism \citep{McSherry2007}, which guarantees $\epsilon$-DP and returns a result randomly from a set of possible candidates, with probability proportional to its ``utility.''
This allows for DP releases of non-numeric information, to which adding numerical noise would be nonsensical.

\vspace{-3pt}\subsection{Privacy-preserving Descriptive Statistics}\vspace{-3pt}

One of the unique aspects of \pkg{DPpack} compared to the other DP R packages is that it provides direct support for DP-satisfying versions of many common descriptive statistics.
The \fn{meanDP}, \fn{varDP}, \fn{covDP}, and \fn{sdDP} functions of \pkg{DPpack} provide DP counterparts to the analogously named R functions for calculating mean, variance, covariance, and standard deviation of a data vector. 
Pooled variances and covariances are also available with \fn{pooledVarDP} and \fn{pooledCovDP}.
Through function arguments, a user specifies whether the output should satisfy $\epsilon$-DP via the Laplace mechanism or $(\epsilon,\delta)$-DP via the Gaussian mechanism and global bounds on the data, from which appropriate $\ell_p$-global sensitivities are computed internally based on known theoretical results \citep{Liu2019Statistical}.

The \fn{histogramDP} and \fn{tableDP} functions compute DP histograms and contingency tables.
Similar to previously described statistics, users may specify which mechanism and type of DP are used for the output, and additional arguments help format the ouptut.
Global bounds on the data are unnecessary as the global sensitivity is a fixed constant for frequency output.

\pkg{DPpack} implements differentially private quantiles and medians using the \fn{quantileDP} and \fn{medianDP} functions, respectively.
By again only requiring the user to input global bounds on the data, these release $\epsilon$-DP values via the exponential mechanism using the private quantile algorithm \citep{Smith2011, Gillenwater2021}.

\vspace{-3pt}\subsection{Privacy-preserving Statistical Models and Machine Learning}\vspace{-3pt}

Empirical risk minimization (ERM) is a statistical learning principle to find the best model from a given set of models.
The goal of ERM is to minimize the empirical risk that measures the goodness of fit of a model to the training data.
We implemented privacy-preserving procedures for a few ERM problems in supervised learning.
Specifically, for binary classification, we create the \fn{EmpiricalRiskMinimizationDP.CMS} class by employing the methods from \citet{chaudhuri2011} for guaranteeing $\epsilon$-DP for the output of training via ERM under necessary regularity conditions.
Either the output or objective perturbation methods can be used.
For linear regression, we employ the methods from \citet{Kifer2012} to create the \fn{EmpiricalRiskMinimizationDP.KST} class for guaranteeing either $\epsilon$-DP or $(\epsilon, \delta)$-DP under necessary regularity conditions. 
The intent is that these classes are used through an inheritance structure to implement binary classifiers or regressors as instances of ERM.

Specifically, logistic regression and support vector machine (SVM) models with $\epsilon$-DP guarantees \citep{Chaudhuri2009, chaudhuri2011} are implemented via the \fn{LogisticRegressionDP} and \fn{svmDP} classes, respectively.
Each of these classes inherits from \fn{EmpiricalRiskMinimizationDP.CMS}.
Released trained model coefficients or predictions made on new data satisfy $\epsilon$-DP.
Linear regression of either $\epsilon$-DP or $(\epsilon, \delta)$-DP \citep{Kifer2012} is implemented in \pkg{DPpack} via the \fn{LinearRegressionDP} class, which inherits from the \fn{EmpiricalRiskMinimizationDP.KST} class.
Released trained model coefficients from those classes or predictions made on new data using these coefficients also satisfy user-specified DP guarantees.
The \fn{svmDP} class currently supports $\epsilon$-DP training via the linear and radial (Gaussian) kernels, with the radial kernel method being based on an approximation technique from \citet{Rahimi2007, Rahimi2008, chaudhuri2011}.
Training with individually weighted loss function contributions with $\epsilon$-DP guarantees is also supported \citep{Giddens2023}.
Each of these methods is user-friendly, even to those without a strong DP background, as they only require the user to specify certain hyperparameters (such as $\epsilon$, $\delta$, and $\gamma$) and global bounds on each feature contained in $\mathbf{x}_i$ (and $y_i$, in the case of linear regression).
Sensitivity calculations and scaling necessary to satisfy regularity conditions for DP guarantees are handled internally.

When the selection of hyperparameter values (e.g., the regularization constant in the ERM loss function) uses information from the sensitive dataset itself, the incurred privacy loss needs to be accounted for.
\pkg{DPpack} provides the \fn{tune\_classification\_model} function for privacy-preserving hyperparameter tuning for binary classifiers based on the exponential mechanism \citep{chaudhuri2011} and the \fn{tune\_linear\_regression\_model} function for hyperparameter tuning for linear regression. 

\vspace{-6pt}\section{Summary and Future Work}\vspace{-3pt}

The \pkg{DPpack} package implements three general mechanisms for DP (Laplace, Gaussian, and exponential), a variety of DP descriptive statistics, and some privacy-preserving regression and classification methods. 
Making these functions accessible independent of the mechanisms they are based on permits code simplicity and ease-of-use (since users do not need to know how to compute sensitivities for their desired statistics, but only need to give global bounds on the data as inputs).
Compared with other options for DP in R, \pkg{DPpack} offers a more complete set of privacy-preserving functions and models in a user-friendly manner that makes them easily accessible even to those without a strong background in DP.

We plan to keep developing the package and make it more comprehensive.
For example, for ML techniques, we may include functionality for DP principal component analysis \citep{Dwork2014, Chaudhuri2013}, Bayesian networks \citep{Zhang2016}, and stochastic gradient descent based on the concepts of moment accountant \citep{Abadi2016} and Gaussian DP \citep{Bu2020}, to name a few.
For statistical analysis, we plan to include functionality for differently private $z$-tests \citep{Gaboardi2019}, $t$-tests \citep{Ding2018}, and some nonparametric tests (e.g. Wilcoxon rank sum test) \citep{Couch2019}, as well as hypothesis testing for linear regression \citep{Barrientos2019,Chen2016} and confidence interval construction for certain problems \citep{Karwa2018, Wang2019}.
We note that the list above is not comprehensive nor are the cited references the only existing work on each respective topic. 

\acks{This work is supported by the University of Notre Dame Schmitt Fellowship and Lucy Graduate Scholarship.}



\appendix

\section{Differential Privacy (DP)}

DP protects the information of each individual whose information is contained in a dataset by ensuring that the results of a mechanism acting on the dataset would be almost identical to the results had their information not been present in the dataset.
To formalize the notion of DP, we first define \textit{neighboring datasets}:
$D_1$ and $D_2$ are \textit{neighboring datasets} if they differ in at most one observation.
There are two equally valid methods by which a neighboring dataset $D_2$ may be constructed from a given dataset $D_1$, depending on if the number of elements of each dataset must remain the same (i.e. is bounded), or if the number is allowed to vary (i.e. is unbounded) \citep{Kifer2011}.
\begin{definition}[Bounded neighboring datasets]
    We consider $D_1$ and $D_2$ to be \textnormal{bounded neighboring datasets} if they are neighboring datasets and $D_1$ can be obtained from $D_2$ by modifying at most one observation.
\end{definition}
\begin{definition}[Unbounded neighboring datasets]
    We consider $D_1$ and $D_2$ to be \textnormal{unbounded neighboring datasets} if they are neighboring datasets and $D_1$ can be obtained from $D_2$ by adding or removing at most one observation.
\end{definition}

The two definitions of neighboring datasets may necessitate different amounts of calibrated noise to achieve the same level of privacy guarantees when releasing the same statistics (e.g, histograms). 
When the sample size is large, the difference between the two is largely ignorable. 

We can now formally define a few different types of differential privacy.
\begin{definition}[Differential privacy] \textnormal{\citep{Dwork2006Calibrating, Dwork2006OurData}}
    A randomized mechanism $\M$ satisfies \textnormal{$(\epsilon,\delta)$-differential privacy} if for all $S\subseteq \textnormal{Range}(\M)$, 
    \begin{equation}
        P(\M(D_1)\in S) \le e^\epsilon P(\M(D_2)\in S) + \delta
    \end{equation} for any neighboring datasets $D_1$ and $D_2$, where $\epsilon>0$ and $\delta\ge0$ are privacy loss parameters.
    It is common to refer to $(\epsilon,0)$-DP (or $\epsilon$-DP) as ``pure'' DP, and $(\epsilon,\delta)$-DP ($\delta>0$) as ``approximate'' DP.
\end{definition}
\begin{definition}[Probabilistic differential privacy] \textnormal{\citep{Machanavajjhala2008}}
     A randomized mechanism $\M$ satisfies \textnormal{$(\epsilon,\delta)$ probabilistic differential privacy} if for all $S\subseteq \textrm{Range}(\M)$, 
    \begin{equation}
        P\bigg(\bigg|\log\bigg(\frac{P(\M(D_1)\in S)}{P(\M(D_2)\in S)}\bigg)\bigg|>\epsilon\bigg) \le \delta
    \end{equation} for any neighboring datasets $D_1$ and $D_2$, $\epsilon>0$ and $\delta\ge0$.
\end{definition}

Intuitively, DP guarantees that the distributions of outputs from a randomized mechanism operating on neighboring datasets are similar.
Thus, information gained from the mechanism will be essentially the same (within tunable bounds given by $\epsilon$ and $\delta$) whether a given individual's data is used in the dataset or not. 
The individual that differs between datasets is arbitrary, meaning that differential privacy provides these individual-level privacy guarantees for all members of the dataset simultaneously.

DP has several nice properties to which its popularity in research and applications is attributed. 
We will briefly mention three here that are relevant to the package, and refer the interested reader to \citet{DPtextbook} for other properties and additional information. 
The first two are composition theorems, which provide differential privacy bounds for the use of multiple randomized mechanisms on the same dataset.
\begin{theorem}[Basic sequential composition] \textnormal{\citep{mcsherry2009privacy}}
    Let $\M_1$, $\M_2,\ldots,\M_n$ be $n$ randomized mechanisms such that each $\M_i$ satisfies $(\epsilon_i, \delta_i)$-differential privacy.
    $\M(D) = (\M_1(D),\ldots,\M_n(D))$ satisfies $(\sum_{i=1}^n\epsilon_i,\sum_{i=1}^n\delta_i)$-differential privacy.
\end{theorem}
\begin{theorem}[Parallel composition] \textnormal{\citep{mcsherry2009privacy}}
    Let $\M_1,\M_2,\ldots,\M_n$ be $n$ randomized mechanisms such that each $\M_i$ satisfies $(\epsilon_i, \delta_i)$-differential privacy, and let $D_1,D_2,\ldots,D_n$ be $n$ disjoint datasets such that their union is $D$. 
    Then we have $\M(D) = (\M_1(D_1), \ldots, \M_n(D_n))$ satisfies $(\max_i\{\epsilon_i\}, \max_i\{\delta_i\})$-differential privacy.
\end{theorem}

The second is immunity to post-processing, which ensures that there is no manipulation (not relying on the data itself) that can be performed on the results of a differentially private mechanism to weaken the privacy guarantees.
\begin{theorem}[Immunity to post-processing] \textnormal{\citep{DPtextbook}}
    Let $\M$ be a randomized mechanism satisfying $(\epsilon,\delta)$-differential privacy. 
    $f\circ\M$ satisfies $(\epsilon,\delta)$-differential privacy for any arbitrary function $f$.
\end{theorem}

We conclude this section by defining the global sensitivity of statistics, which is used in some of the most general randomized mechanisms for achieving DP. 
Global sensitivity was originally defined using the $\ell_1$ norm by \citet{Dwork2006Calibrating}.
Here, we use a more general definition.
\begin{definition}[$\ell_p$-global sensitivity] \textnormal{\citep{Liu2019Generalized}}
    Let the distance between two datasets (denoted $d(D_1,D_2)$) be defined to be the number of observations that differ between the datasets.
    Note that $d(D_1,D_2)=1$ if $D_1$ and $D_2$ are neighboring datasets.
    The \textnormal{$\ell_p$-global sensitivity} of a function $f$ is defined to be
    \begin{equation}
        \Delta_{p,f} = \max_{\substack{D_1, D_2 \\ d(D_1,D_2) = 1}}||f(D_1) - f(D_2)||_p,
    \end{equation}
    where $||\cdot||_p$ is the $\ell_p$ norm.
\end{definition}

The global sensitivity may be different depending on if the bounded or the unbounded neighboring dataset definition is used.
For example, consider the function that outputs a histogram (a list of counts for each bin) from a given dataset.
If the bounded neighboring dataset definition is used, the $\ell_1$-global sensitivity of the histogram function is $2$ since modifying a dataset observation can at most change the count of two bins.
However, under the unbounded neighboring dataset definition, the $\ell_1$-global sensitivity of the histogram function is $1$ since adding or removing a dataset observation can at most change the count of one bin.
For functions where there may be a difference between the two definitions of neighboring datasets, \pkg{DPpack} allows the user to choose which one to use.

The global sensitivity sets a bound on the amount the statistics can change in the worst-case scenario between two neighboring data sets. 
The higher the sensitivity for a statistic is, the larger the amount of noise that will be injected to the original observed statistics to achieve the pre-specified level of privacy guarantees defined by $\epsilon$ and $\delta$.

For many statistics (e.g., mean, quantiles, regression coefficients), the global sensitivity is dependent on the global range of values that can occur in the dataset. 
In these cases, \pkg{DPpack} assumes the existence of known or reasonably inferred global or public bounds on the dataset, from which the global sensitivity is computed.

\section{DP Mechanisms}

There exist many general mechanisms for ensuring DP for a given analysis procedure or output. 
We introduce in this section three popular mechanisms: the Laplace mechanism \citep{Dwork2006Calibrating}, the Gaussian mechanism \citep{Dwork2006OurData}, and the exponential mechanism \citep{McSherry2007}.
We also provide examples of implementing these mechanisms using \pkg{DPpack}.

\subsection{Laplace Mechanism}

\begin{definition}[Laplace mechanism] \textnormal{\citep{Dwork2006Calibrating}}
    Let $D$ be a sensitive database. 
    Let $f$ be a given function with $\ell_1$-global sensitivity $\Delta_{1,f}$ and range $\mathbb{R}^n$.
    The \textnormal{Laplace mechanism} of $\epsilon$-differential privacy is defined to be
    \begin{equation}
        \M_L(D, f, \epsilon) = f(D) + \e,
    \end{equation}
    where $\e=(e_1,\ldots,e_n)^T$ and $e_i$ is drawn independently from distribution Lap($0,\Delta_{1,f}/\epsilon$).
\end{definition}

The \fn{LaplaceMechanism} function in \pkg{DPpack} implements the Laplace mechanism.
For a given scalar or a vector of observed statistic(s), the corresponding $\ell_1$-global sensitivity, and $\epsilon$, it releases a real number or numeric vector of values satisfying $\epsilon$-DP.
Global sensitivity calculated based either on bounded or unbounded neighboring datasets can be used. 

The following example uses the Laplace mechanism to release the sample mean with $\epsilon$-DP guarantees. 
Consider a sensitive dataset of $n=100$ observations with one attribute, the (public) global range of which is $[c_0,c_1] = [5,10]$. 
For the sample mean, the $\ell_1$-global sensitivity is the same for both bounded and unbounded DP and equals $(c_1-c_0)/n = 0.05$ \citep{Liu2019Statistical}.

\begin{example}
    library(DPpack)
    set.seed(42) # For reproducibility
    # Simulate a dataset
    n <- 100
    c0 <- 5
    c1 <- 10
    D <- runif(n, c0, c1)
    epsilon <- 1 # Privacy budget
    sensitivity <- (c1-c0)/n
    
    private.mean <- LaplaceMechanism(mean(D), epsilon, sensitivity)
    cat("Privacy preserving mean: ", private.mean, "\nTrue mean: ", mean(D))
    #> Privacy preserving mean:  7.636944 
    #> True mean:  7.622394
\end{example}

The \fn{LaplaceMechanism} function can also be used to release privacy-preserving multi-dimensional statistics which are the composition of scalar statistics, each with their own $\ell_1$ sensitivity. 
For example, let $\mathbf{f}(D) = (f_1(D),\ldots,f_n(D))$, where each $f_i$ has $\ell_1$-global sensitivity $\Delta_{1,f_i}$.
By default, the \fn{LaplaceMechanism} function sanitizes $\mathbf{f}$ by drawing Laplace noise from Lap($0,\Delta_{1,\mathbf{f}}/\epsilon)$, where $\Delta_{1,\mathbf{f}} = \sum_{i=1}^n\Delta_{1,f_i}$.
This approach corresponds to allocating a privacy budget of $\epsilon\Delta_{1,f_i}/\Delta_{1,\mathbf{f}}$ to sanitizing each scalar function $f_i$.
If desired, users may specify how to divide the total budget $\epsilon$ among the elements in $\mathbf{f}$ by passing a vector of proportions to the \code{alloc.proportions} argument instead of using the default allocation.
The following example demonstrates this functionality for the same situation as the previous example, but with an additional variance computation. 
The $\ell_1$-global sensitivity of the variance is also the same for both bounded and unbounded DP and equals $(c_1-c_0)^2/n = 0.25$ \citep{Liu2019Statistical}.

\begin{example}
    # Simulate a dataset
    n <- 100
    c0 <- 5
    c1 <- 10
    D <- runif(n, c0, c1)
    f <- function(D) c(mean(D), var(D))
    sensitivities <- c((c1-c0)/n, (c1-c0)^2/n)
    epsilon <- 1 # Total privacy budget for f
    
    # Here, privacy budget is split relative to the individual sensitivities
    # of the sample mean and sample variance. Collectively, the computation
    # satisfies 1-differential privacy.
    private.vals <- LaplaceMechanism(f(D), epsilon, sensitivities)
    cat("Privacy preserving values: ", private.vals, "\nTrue values: ", f(D))
    #> Privacy preserving values:  7.623156 2.401604
    #> True values:  7.61271 2.036525
    
    # Here, privacy budget is split so that 25% is given to the mean
    # and 75% is given to the variance
    private.vals <- LaplaceMechanism(f(D), epsilon, sensitivities,
                                     alloc.proportions = c(0.25, 0.75))
    cat("Privacy preserving values: ", private.vals, "\nTrue values: ", f(D))
    #> Privacy preserving values:  7.58841 1.652268
    #> True values:  7.61271 2.036525
\end{example}

\subsection{Gaussian Mechanism}

Another popular mechanism for DP implemented in \pkg{DPpack} is the Gaussian mechanism.
This mechanism can be used to provide either $(\epsilon,\delta)$ approximate DP \citep{Dwork2006OurData} or $(\epsilon,\delta)$ probabilistic DP \citep{Machanavajjhala2008}.
\begin{definition}[Gaussian mechanism] \textnormal{\citep{Dwork2006OurData}}
    Let $D$ be a sensitive database. 
    Let $f$ be a given function with $\ell_2$-global sensitivity $\Delta_{2,f}$ and range $\mathbb{R}^n$.
    The \textnormal{Gaussian mechanism} is defined to be
    \begin{equation}
        \M_G(D, f, \epsilon, \delta) = f(D) +  \e,
    \end{equation}
    where $\e=(e_1,\ldots,e_n)^T$ and $e_i$ is drawn independently  from $\mathcal{N}(0, \sigma^2)$. In the case that $\epsilon \in (0,1)$ and 
    \begin{equation}
        \sigma\ge c\Delta_{2,f}/\epsilon
    \end{equation} 
    for a constant $c$ such that $c^2>2\log(1.25/\delta)$, this mechanism was proven to satisfy approximate $(\epsilon,\delta)$-DP \citep{Dwork2006OurData}.
    Additionally, when
    \begin{equation}
        \sigma \ge (2\epsilon)^{-1}\Delta_{2,f}\bigg(\sqrt{(\Phi^{-1}(\delta/2))^2 + 2\epsilon} - \Phi^{-1}(\delta/2)\bigg),
    \end{equation}
    where $\Phi$ is the CDF of the standard normal distribution, this mechanism was proven to satisfy $(\epsilon, \delta)$ probabilistic DP \citep{Liu2019Generalized}.
\end{definition}

Note the requirement that $\epsilon<1$ for approximate DP, which is not required for the Gaussian mechanism to satisfy probabilistic DP. 
It is also worth highlighting that the Laplace mechanism requires $\ell_1$-sensitivity, while the Gaussian mechanism requires $\ell_2$-sensitivity. 
If $f$ is scalar-valued, $\Delta_{1,f} = \Delta_{2,f}$, but they are generally different for vector-valued $f$ except in some special cases.

The \code{GaussianMechanism} function in \pkg{DPpack} implements the Gaussian mechanism by adding Gaussian noise to a given scalar (or vector) of observed statistic(s) according to specified values of $\epsilon$, $\delta$, and $\ell_2$-global sensitivity.
It releases a scalar (or vector) satisfying either $(\epsilon,\delta)$ approximate DP if the \code{type.DP} argument is \code{'aDP'}, or $(\epsilon, \delta)$ probabilistic DP if the \code{type.DP} argument is \code{'pDP'}. 
Global sensitivity calculated based either on bounded or unbounded neighboring datasets can be used. 

We use the same example as for the Laplace mechanism to demonstrate the Gaussian mechanism for $(\epsilon, \delta)$ approximate DP and $(\epsilon, \delta)$ probabilistic DP for a sample mean.
Consider again a sensitive dataset of $n=100$ elements drawn uniformly from the range $[c_0,c_1] = [5,10]$.
Since the mean is a scalar in this case, the $\ell_2$-global sensitivity is equal to the $\ell_1$-global sensitivity, which is $(c_1 - c_0)/n = 0.05$ \citep{Liu2019Statistical}.

\begin{example}
    # Simulate a dataset
    n <- 100
    c0 <- 5
    c1 <- 10
    D <- runif(n, c0, c1)
    
    # Privacy budget
    epsilon <- 0.9 # eps must be in (0, 1) for approximate DP
    delta <- 0.01
    sensitivity <- (c1-c0)/n
    
    # Approximate differential privacy
    private.approx <- GaussianMechanism(mean(D), epsilon, 
                                        delta, sensitivity)
    cat("Privacy-preserving mean (approximate): ", private.approx, 
        "\nTrue mean: ", mean(D))
    #> Privacy preserving mean (approximate):  7.426412
    #> True mean:  7.170852
    
    # Probabilistic differential privacy
    private.prob <- GaussianMechanism(mean(D), epsilon, delta, 
                                      sensitivity, type.DP = 'pDP')
    cat("Privacy preserving mean (probabilistic): ", private.prob, 
        "\nTrue mean: ", mean(D))
    #> Privacy-preserving mean (probabilistic):  7.018747
    #> True mean:  7.170852
\end{example}

The \code{GaussianMechanism} function can also be used to release privacy-preserving multi-dimensional statistics analogously to the \code{LaplaceMechanism} function with only one difference.
If we again consider $\mathbf{f}(D) = (f_1(D),\ldots,f_n(D))$ to be the multi-dimensional statistics of interest, then $\Delta_{2,\mathbf{f}}$ for the Gaussian mechanism is computed as $\Delta_{2,\mathbf{f}} = \sqrt{\sum_{i=1}^n\Delta_{2,f_i}^2}$ by default.
If desired, users can specify their own privacy budget allocation (which applies to both $\epsilon$ and $\delta$) using the \code{alloc.proportions} argument.

\subsection{Exponential Mechanism}

The third privacy-preserving mechanism implemented in \pkg{DPpack} is the exponential mechanism, developed in \citet{McSherry2007}.
This mechanism is preferred for situations where it is not possible to inject numerical noise (such as when the function output is categorical) or not appropriate to add noise directly to the result of a given function or algorithm.
The exponential mechanism resolves this issue by assigning real-valued utilities to data/output pairs by specifying a utility function $u$.
An output is chosen and released with probability proportional to its corresponding utility.

\begin{definition}[Exponential mechanism] \textnormal{\citep{McSherry2007}}
    Let $D$ be a sensitive database, $f$ be a given function with range $\mathcal{R}$, and $u$ be a utility function mapping data/output pairs to $\mathbb{R}$ with $\ell_1$-global sensitivity $\Delta_{1,u}$.
    For output values $r \in \mathcal{R}$, the \textnormal{exponential mechanism} achieving $\epsilon$-DP is 
    \begin{equation}
        \M_E(D, u, \mathcal{R}, \epsilon) = r \textnormal{ with probability } \propto\exp{\left(\frac{\epsilon u(D,r)}{2\Delta_{1,u}}\right)}.
        \label{eq:exponential_mechanism}
    \end{equation}
\end{definition}

The \code{ExponentialMechanism} function in \pkg{DPpack} implements the exponential mechanism for differential privacy for a given sensitive dataset $D$ and for finite $\mathcal{R}$.
It takes as input a numeric vector \code{utility} representing the values of the utility function $u$ for each $r\in\mathcal{R}$, as well as a privacy budget $\epsilon$ and the $\ell_1$-global sensitivity of $u$.
It releases the index corresponding to the value $r\in\mathcal{R}$ randomly selected according to \eqref{eq:exponential_mechanism}.
Global sensitivity of $u$ calculated based either on bounded or unbounded neighboring datasets can be used. 

The \code{ExponentialMechanism} function also has two optional arguments: \code{measure} and \code{candidates}.
Each of these arguments, if provided, should be of the same length as \code{utility}.
If \code{measure} is given, the probabilities of selecting each value $r$ are weighted according to the numeric values in \code{measure} before the value $r$ is randomly chosen.
If \code{candidates} is provided, \code{ExponentialMechanism} returns the value in \code{candidates} at the randomly chosen index rather than the index itself.

We demonstrate the \code{ExponentialMechanism} function with a toy example.
Assume that a function $f$ has range $\mathcal{R} = \{$\code{`a'}, \code{`b'}, \code{`c'}, \code{`d'}, \code{`e'}$\}$. 
Numerical noise cannot be added directly to the output of $f$ due to the non-numeric nature of its range. 
Instead, we define a utility function $u$ that yields the following values when applied to the sensitive dataset $D$ and each element of $\mathcal{R}$, respectively: $(0, 1, 2, 1, 0)$.
Finally, assume the $\ell_1$-sensitivity of $u$ is 1.
We can use the \code{ExponentialMechanism} function to release an element of $\mathcal{R}$ as follows.

\begin{example}
    candidates <- c(`a', `b', `c', `d', `e') # Range of f
    # Utility function values in same order as corresponding candidates
    utility <- c(0, 1, 2, 1, 0)
    epsilon <- 1 # Privacy budget
    sensitivity <- 1
    
    # Release privacy-preserving index of chosen candidate
    idx <- ExponentialMechanism(utility, epsilon, sensitivity)
    candidates[idx]
    #> `b'
    
    # Release privacy-preserving candidate directly
    ExponentialMechanism(utility, epsilon, sensitivity, 
                         candidates = candidates)
    #> `a'
\end{example}

\section{Implementation of DP Descriptive Statistics}

Descriptive statistics are popular and effective ways to summarize data.
However, if these statistics are computed from a sensitive dataset and released directly, they could be susceptible to attacks that reveal private information about the individuals in the data, even if the dataset itself is not breached.
Many of these statistics can be made differentially private through the application of one or more of the mechanisms discussed in the previous section.
For ease of use, \pkg{DPpack} implements privacy-preserving versions of many descriptive statistics directly, utilizing the previously defined mechanisms under the hood.

\subsection{Mean, Standard Deviation, Variance, and Covariance}

The \code{meanDP},  \code{sdDP}, and \code{varDP}, functions can be used to release differentially private means, standard deviations, and variances respectively, calculated from a sensitive dataset. 
These functions all share the same set of arguments: a dataset \code{x}, a privacy budget \code{eps} (and possibly \code{delta}), as well as bounds on the attributes in the dataset \code{lower.bound} and \code{upper.bound}.
Any values of \code{x} that happen to fall outside the bounds are clipped to the bounds before the mean is computed.
These bounds are used to compute the global sensitivity of the desired statistic function based on proven values \citep{Liu2019Statistical}.

By default, each function releases sanitized values satisfying \code{eps}-DP via the Laplace mechanism. 
The \code{mechanism} argument defaults to \code{`Laplace'}, indicating to use the Laplace mechanism.
However, the output can be changed by modifying the value of some additional arguments and setting \code{mechanism} to \code{`Gaussian'}.
In this case, the \code{delta} argument must be positive.
The \code{type.DP} argument can be either \code{`aDP'} (default) or \code{`pDP'} for satisfying $($\code{eps}, \code{delta}$)$ approximate DP and $($\code{eps}, \code{delta}$)$ probabilistic DP, respectively, and indicates the type of DP provided when the Gaussian mechanism is used.
The \code{which.sensitivity} argument can be one of \code{`bounded'} (default), \code{`unbounded'}, or \code{`both'}, indicating whether to release results satisfying bounded and/or unbounded DP.
The following example demonstrates how these functions can be used.

\begin{example}
    # Simulate a dataset
    D <- rnorm(500, mean=3, sd=2)
    lower.bound <- -3 # 3 standard deviations below mean
    upper.bound <- 9 # 3 standard deviations above mean
    
    # Get mean satisfying bounded 1-differential privacy
    private.mean <- meanDP(D, 1, lower.bound, upper.bound)
    cat("Privacy preserving mean: ", private.mean, "\nTrue mean: ", mean(D))
    #> Privacy preserving mean:  2.872637
    #> True mean:  2.857334
    
    # Get variance satisfying unbounded approximate (0.5, 0.01)-DP
    private.var <- varDP(D, 0.5, lower.bound, upper.bound,
                         which.sensitivity = `unbounded', 
                         mechanism = `Gaussian', delta = 0.01)
    cat("Privacy preserving variance: ", private.var,
    "\nTrue variance: ", var(D))
    #> Privacy preserving variance:  3.276551
    #> True variance:  4.380399
    
    # Get std dev satisfying bounded probabilistic (0.5, 0.01)-DP
    private.sd <- sdDP(D, 0.5, lower.bound, upper.bound, 
                       mechanism=`Gaussian', delta=0.01, type.DP=`pDP')
    cat("Privacy preserving standard deviation: ", private.sd,
        "\nTrue standard deviation: ", sd(D))
    #> Privacy preserving standard deviation:  1.978296
    #> True standard deviation:  2.09294
\end{example}

The \code{pooledVarDP} function in \pkg{DPpack} can be used to compute a differentially private pooled variance for multiple groups of data.
The inputs are similar to those of \code{meanDP}, \code{varDP}, and \code{sdDP} with a few differences.
First, the function accepts multiple numeric vectors representing different data groups, rather than a single dataset \code{x}.
The function uses provided lower and upper bounds on the entire collection of data to compute the sensitivity of the function based on the derived formulas in \citet{Liu2019Statistical}, then releases a privacy preserving pooled variance of the entire collection of data based on provided privacy budget parameters.
The formulas to compute the function's sensitivity require a value $n_{\textnormal{max}}$ representing the size of the largest provided dataset vector.
If the value itself is sensitive, it can be approximated by setting the \code{approx.n.max} argument to \code{TRUE}.
The following examples demonstrate this function's use.

\begin{example}
    # Simulate three datasets from the same distribution
    D1 <- rnorm(500, mean=3, sd=2)
    D2 <- rnorm(200, mean=3, sd=2)
    D3 <- rnorm(100, mean=3, sd=2)
    lower.bound <- -3 # 3 standard deviations below mean
    upper.bound <- 9 # 3 standard deviations above mean
    
    # Get private pooled variance without approximating n.max
    private.pooled.var <- pooledVarDP(D1, D2, D3, eps=1,
                                      lower.bound = lower.bound,
                                      upper.bound = upper.bound)
    cat("Privacy preserving pooled variance: ", private.pooled.var,
        "\nTrue pooled variance: ", var(c(D1, D2, D3)))
    #> Privacy preserving pooled variance:  3.682308 
    #> True pooled variance:  3.931237
    
    # If n.max is sensitive, we can also use
    private.pooled.var <- pooledVarDP(D1, D2, D3, eps=1, 
                                      lower.bound = lower.bound,
                                      upper.bound = upper.bound, 
                                      approx.n.max = FALSE)
\end{example}

\pkg{DPpack} also implements functions for privacy-preserving covariance and pooled covariance: \code{covDP} and \code{pooledCovDP}, which have similar arguments to the previously described functions.
The \code{covDP} function accepts two numeric vector datasets \code{x1} and \code{x2}, as well as upper and lower bounds on each of these two datasets individually. 
The function then returns the sanitized covariance between \code{x1} and \code{x2}, based on provided privacy budget values and sensitivity computed using the bounds via the proven formula from \citet{Liu2019Statistical}.
The \code{pooledCovDP} function accepts any number of matrices.
These matrices can have a variable number of rows, but must each have two columns.
Two sets of bounds for the entire collection of data from each column must also be provided.
The function releases a sanitized pooled covariance between the columns of the provided matrices based on the privacy budget, bounds, and the sensitivity computed according to the formula from \citet{Liu2019Statistical}.
Finally, \code{pooledCovDP} utilizes the value $n_{\textnormal{max}}$ in the computation of the sensitivity similar to the \code{pooledVarDP} function, so the \code{approx.n.max} argument is also present in this function and indicates the same thing.
The following examples show the use of both of these functions.

\begin{example}
    # Simulate datasets
    D1 <- sort(rnorm(500, mean=3, sd=2))
    D2 <- sort(rnorm(500, mean=-1, sd=0.5))
    lb1 <- -3 # 3 std devs below mean
    lb2 <- -2.5 # 3 std devs below mean
    ub1 <- 9 # 3 std devs above mean
    ub2 <- .5 # 3 std devs above mean
    
    # Covariance satisfying 1-differential privacy
    private.cov <- covDP(D1, D2, 1, lb1, ub1, lb2, ub2)
    cat("Privacy preserving covariance: ", private.cov,
        "\nTrue covariance: ", cov(D1, D2))
    #> Privacy preserving covariance:  0.9598711 
    #> True covariance:  0.9908612
    
    # We can also find a sanitized pooled covariance with additional datasets
    D3 <- sort(rnorm(200, mean=3, sd=2))
    D4 <- sort(rnorm(200, mean=-1, sd=0.5))
    M1 <- matrix(c(D1, D2), ncol=2)
    M2 <- matrix(c(D3, D4), ncol=2)
    
    # Pooled covariance satisfying (1,0)-differential privacy
    private.pooled.cov <- pooledCovDP(M1, M2, eps = 1, lower.bound1 = lb1,
                                    lower.bound2 = lb2, upper.bound1 = ub1,
                                    upper.bound2 = ub2)
\end{example}

\subsection{Counting Functions}

\pkg{DPpack} supports differentially private histograms and contingency tables via the functions \code{histogramDP} and \code{tableDP}, respectively. 
The functions release privacy-preserving results based on given sensitive input data (in the same form required by the standard \code{hist} and \code{table} functions) and privacy budget parameters.
Bounds on the dataset are not necessary as the global sensitivity for both functions is a constant independent of the data.
As with many of the previously described functions, the guaranteed DP for both of these functions can be bounded or unbounded, as well as pure, approximate, or probabilistic depending on the values given for the \code{which.sensitivity}, \code{mechanism}, and \code{type.DP} arguments.
Due to noise added to the typical output by both the Laplace and Gaussian mechanisms, it is possible that some counts obtained directly from the chosen mechanism are negative.
By default, both of these functions coerce any such values to 0.
However, if in a particular application it is preferred that negative counts be allowed, this can be done by setting the \code{allow.negative} argument to \code{TRUE}.
The \code{histogramDP} function has two additional arguments: \code{breaks} and \code{normalize}.
The \code{breaks} argument is equivalent to the argument with the same name in the standard \code{hist} function, while the \code{normalize} argument indicates whether the outputs should correspond to frequencies (if set to \code{FALSE}) or if they should be normalized so that the total area under the histogram is 1 (if set to \code{TRUE}).

The following examples demonstrate the proper use of the \code{histogramDP} and \code{tableDP} functions.
Note that \code{histogramDP} returns an object similar to that returned by the standard \code{hist} function, but does not plot the histogram by default. 
Plotting the result is as easy as calling the \code{plot} function on the object released from \code{histogramDP}.
The results are shown in Figure \ref{fig:histogram example}.
\begin{example}
    x <- rnorm(500) # Simulate dataset
    hist(x, main = "Non-private histogram", ylim=c(0, 110), col="gray")
    private.hist <- histogramDP(x, 1) # Satisfies (1,0)-DP
    plot(private.hist, main = "Private histogram", 
         ylim=c(0, 110), col="gray")
\end{example}

\begin{figure}[!htb]
\centering
\begin{subfigure}[b]{0.49\textwidth}
\centering
\includegraphics[width=\textwidth, trim=3pt 18pt 12pt 12pt, clip]{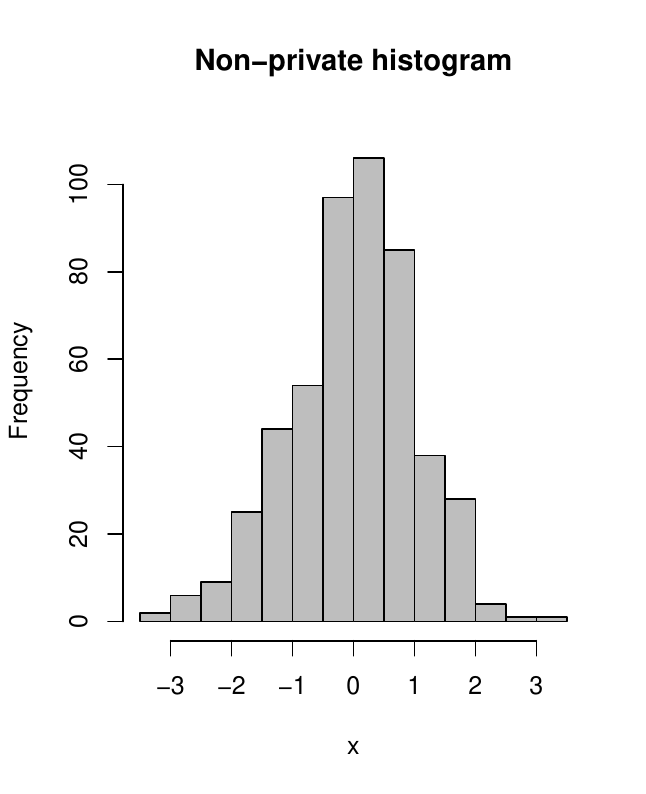}
\label{fig:non-private histogram}
\end{subfigure}
\begin{subfigure}[b]{0.49\textwidth}
\centering
\includegraphics[width=\textwidth, trim=3pt 18pt 12pt 12pt, clip]{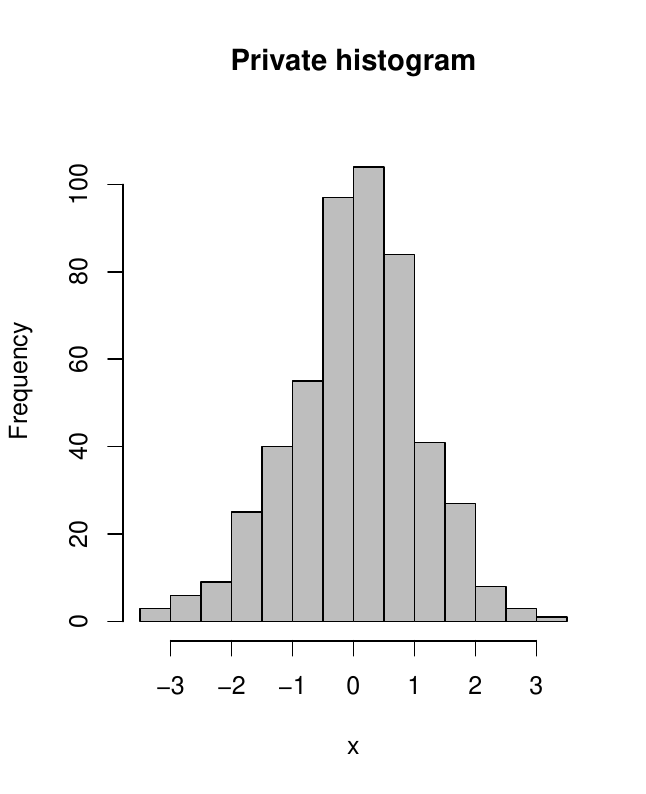}
\label{fig:private histogram}
\end{subfigure}\vspace{-9pt}
\caption{Original and privacy-preserving histograms from the histogram example.}
\label{fig:histogram example}
\end{figure}

We use a subset of variables from the \code{Cars93} dataset in the \pkg{MASS} R package to demonstrate the generation of privacy-preserving contingency table. 
The results are shown in Table \ref{tab:contingency_table_example}.
\begin{example}
    x <- MASS::Cars93$Type
    y <- MASS::Cars93$Origin
    z <- MASS::Cars93$AirBags
    table(x, y, z) # Non-private contingency table
    tableDP(x, y, z, eps=1) # Private contingency table
\end{example}

\begin{table}[!htb]
\centering
\caption{The outputs from the contingency table example.
The privacy-preserving cell counts are listed with the original (non-private) values in parentheses.}
\label{tab:contingency_table_example}
\resizebox{0.85\textwidth}{!}{
\begin{tabular}{l l c c c c c c}
    \hline
    Airbag& Origin &\multicolumn{6}{c}{Type}\\
    \cline{3-8}
    && Compact & Large & Midsize & Small & Sporty & Van\\
    \hline
    Driver & USA     & 0 (1) & 4 (4) & 7 (2) & 0 (0) & 1 (2) & 0 (0) \\
     \& Passenger   & non-USA & 1 (1) & 0 (0) & 5 (5) & 0 (0) & 1 (1) & 0 (0) \\
    \hline
    Driver only         & USA     & 0 (2) & 5 (7) & 6 (5) & 2 (2) & 11 (5) & 3 (2) \\
                        & non-USA & 7 (7) & 4 (0) & 4 (6) & 2 (3) & 3 (3) & 1 (1) \\
    \hline
    None                & USA     & 4 (4) & 0 (0) & 6 (3) & 4 (5) & 0 (1) & 2 (3) \\
                        & non-USA & 0 (1) & 1 (0) & 3 (1) & 16 (11) & 2 (2) & 7 (3) \\
    \hline
\end{tabular}}
\end{table}

\subsection{Quantiles}

\pkg{DPpack} also implements differentially private quantiles and medians using the \code{quantileDP} and \code{medianDP} functions, respectively.
The \code{quantileDP} function accepts a sensitive dataset as a numeric vector, a real number between 0 and 1 indicating the desired quantile, a single privacy budget parameter \code{eps}, and global bounds on the values in the dataset.
It implements the private quantile algorithm from \citet{Smith2011}, which defines a utility function, and utilizes the exponential mechanism to release a quantile satisfying $\epsilon$-DP based on the proven $\ell_1$-global sensitivity of the utility function \citep{Smith2011, Gillenwater2021}. 
The algorithm from \citet{Smith2011} used in \code{quantileDP} uses the exponential mechanism to select a specific dataset value from the given dataset but releases a value drawn uniformly from the interval between the selected value and the subsequent value in ascending order.
This means that the released value may not necessarily be a value present in the original dataset.
If this behavior is not desirable for a certain application, the \code{uniform.sampling} argument can be set to \code{FALSE}, in which case the function releases the result of the exponential mechanism step directly without the uniform sampling step.
The \code{medianDP} function is present in \pkg{DPpack} for convenience, and works identically to the \code{quantileDP} function with the quantile argument set to 0.5. 

Both functions accept two additional arguments.
The \code{which.sensitivity} argument operates analogously to the identically named argument in the other functions described in this section.
The \code{mechanism} argument indicates which mechanism should be used to satisfy DP when running the function.
Currently, only the exponential mechanism (the default for this argument) is supported for \code{quantiileDP}, but this argument was still included for symmetry with the other descriptive statistic functions, as well as for robustness in future versions of \code{DPpack}.
The following examples show the use of both of these functions.

\begin{example}
    # Simulate a dataset
    D <- rnorm(500)
    lower.bound <- -3 # 3 standard deviations below mean
    upper.bound <- 3 # 3 standard deviations above mean
    
    quant <- 0.25
    eps <- 1
    # Get 25th quantile satisfying 1-differential privacy
    private.quantile <- quantileDP(D, quant, eps, lower.bound, upper.bound)
    cat("Privacy preserving quantile: ", private.quantile,
        "\nTrue quantile: ", quantile(D, 0.25))
    #> Privacy preserving quantile:  -0.7768781 
    #> True quantile:  -0.7685687
    
    # Get median requiring released value to be in dataset
    private.median <- medianDP(c(1,0,3,3,2), eps, lower.bound = 0, 
                               upper.bound = 4,
                               uniform.sampling = FALSE)
    cat("Privacy preserving median: ", private.median,
        "\nTrue median: ", median(c(1,0,3,3,2)))
    #> Privacy preserving median:  1
    #> True median:  2
\end{example}

\section{Implementation of DP Statistical and ML Methods}

\pkg{DPpack} implements privacy-preserving versions of some commonly used  classification and regression models.
Many such models can be formulated as empirical risk minimization (ERM) problems, which have been generally shown to have privacy-preserving counterparts under certain assumptions \citep{chaudhuri2011, Kifer2012}.
This section first provides a brief introduction to differentially private ERM algorithms and their necessary assumptions, then discusses the specific implementation in \pkg{DPpack} of logistic regression, support vector machines (SVM) and their extension to outcome weighted learning (OWL), and linear regression.

Each of the ERM-based methods implemented in \pkg{DPpack} requires the selection of various hyperparameter values that can impact model performance.
A variety of techniques exist to tune these parameters, but many of these techniques threaten to leak private database information themselves.
Thus, privacy-preserving hyperparameter tuning methods for both the classification and regression models are implemented in \pkg{DPpack}.
These are also described in this section.

\subsection{Empirical Risk Minimization}

Assume that we have a set of $n$ input-output pairs $(\mathbf{x}_i, y_i) \in (\mathcal{X},\mathcal{Y})$ representing a sensitive training dataset $\D$.
Additionally, define $\ell: \mathcal{Y} \times \mathcal{Y} \rightarrow \mathbb{R}$ to be a loss function over pairs of values from the output space.
In general, ERM attempts to produce an effective predictor function $f:\mathcal{X} \rightarrow \mathcal{Y}$ by minimizing the empirical risk
\begin{equation}
    \frac{1}{n}\sum_{i=1}^n\ell(f_{\boldsymbol{\theta}}(\mathbf{x}_i), y_i) = \frac{1}{n}\sum_{i=1}^n\ell_i(\boldsymbol{\theta}).
\end{equation}
For the algorithms implemented in \pkg{DPpack}, we assume there exists a one-to-one mapping from a $p$-dimensional vector $\boldsymbol{\theta}$ to $f$, where $p$ is the length of $\mathbf{x}_i$ (i.e. the number of predictors).
In order to mitigate overfitting, it is also common to introduce a regularizer function $R$. 
This produces the regularized ERM model
\begin{equation}\label{eq:regularized_ERM}
    \frac{1}{n}\sum_{i=1}^n \ell(f_{\boldsymbol{\theta}}(\mathbf{x}_i), y_i) + \frac{\gamma}{n}R(\boldsymbol{\theta}) = \ell_i(\boldsymbol{\theta}) + \frac{\gamma}{n}R(\boldsymbol{\theta}),
\end{equation}
where $\ell_i(\boldsymbol{\theta}) \!=\! \ell(f_{\boldsymbol{\theta}}(\mathbf{x}_i), y_i)$ and $\gamma$ is a tunable hyperparameter known as the regularization constant.

For binary classification problems, \citet{chaudhuri2011} proved $\epsilon$-DP can be satisfied for regularized ERM by two different algorithms if certain assumptions are met. 
The first algorithm is an \emph{output} perturbation method, and the second is an \emph{objective} perturbation method. 
We briefly mention the assumptions here and refer the interested reader to \citet{chaudhuri2011} for more information and for the proofs.
Both algorithms assume $\norm{\mathbf{x}_i}_2\le 1$ for all $i$, that the regularizer $R$ is differentiable and 1-strongly convex, and that the loss function $\ell$ is differentiable and convex with $\big|\frac{\partial}{\partial f}\ell(f, y)\big|\le 1$ for all $f$ and $y$. 
The objective perturbation algorithm has additional assumptions that $R$ and $\ell$ are doubly differentiable and that $\big|\frac{\partial^2}{\partial f^2}\ell(f, y)\big|\le c$ for some constant $c$.

The output perturbation method first solves Eqn \eqref{eq:regularized_ERM} and then perturbs the output $\hat{\boldsymbol{\theta}}$ by adding noise determined by the values of $n$, $\epsilon$, and $\gamma$.
The objective function perturbation method adds random noise determined by $n$, $\epsilon$, $\gamma$, and $c$ directly to the objective function, then finds $\hat{\boldsymbol\theta}$ minimizing the perturbed function.
This amounts to the privacy-preserving regularized ERM model
\begin{equation}\label{eq:private_regularized_ERM}
    \frac{1}{n}\sum_{i=1}^n \ell_i(\boldsymbol\theta) + \frac{\gamma}{n}R(\boldsymbol\theta) + \frac{\Delta}{2n}||\boldsymbol\theta||^2_2 + \frac{\mathbf{b}^T\boldsymbol\theta}{n},
\end{equation}
where $\mathbf{b}$ is the injected random noise and $\frac{\Delta}{2n}\norm{\boldsymbol\theta}^2_2$ is an additional slack term necessary for DP via objective perturbation.
\citet{chaudhuri2011} show that objective perturbation generally provides better utility guarantees than output perturbation for the same privacy budget.

\pkg{DPpack} implements both algorithms using the \code{EmpiricalRiskMinimizationDP.CMS} \pkg{R6} class.
This class provides a general framework for running these algorithms, but is not intended to be utilized directly.
Rather, it should be used as the parent class in an inheritance structure where the child class implements a specific realization of ERM for binary classification (i.e. logistic regression).
Examples of this will be discussed in the subsequent sections.

For regression problems, \citet{Kifer2012} proposed a slightly different algorithm that satisfies DP for regularized ERM in Eqn \eqref{eq:private_regularized_ERM}.
Assume that $\ell_i(\boldsymbol\theta)$ is convex with a continuous Hessian, $R$ is convex, and the following conditions hold for all $\mathbf{x}_i, y_i$ and for all $\boldsymbol\theta\in\mathbb{F}$ ($\mathbb{F}$ is a closed convex subset of $\mathbb{R}^p$): $\norm{\nabla_{\boldsymbol\theta} \ell_i(\boldsymbol\theta)}_2 \le \zeta$ for some constant $\zeta$, the eigenvalues of $\nabla^2_{\boldsymbol\theta}\ell_i(\boldsymbol\theta)$ are bounded above by some constant $\lambda$, and the rank of $\nabla^2_{\boldsymbol\theta}\ell_i(\boldsymbol\theta)$ is at most one. 
Then the solutions $\hat{\boldsymbol\theta}\in\mathbb{F}$ from minimizing the perturbed objective Eqn \eqref{eq:private_regularized_ERM} satisfy DP\footnote{Though the objective perturbation algorithms from \citet{chaudhuri2011} and \citet{Kifer2012} can both be written in the general form of Eqn \eqref{eq:private_regularized_ERM}, it is worth emphasizing that the former requires $R$ to be 1-strongly convex, while the latter only requires $R$ to be convex.
The popular $\ell_1$ regularizer is an example of a convex, but not 1-strongly convex regularizer.}. The algorithm can be used to satisfy either $\epsilon$-DP or approximate $(\epsilon, \delta)$-DP.
If $\epsilon$-DP is desired, the noise vector is drawn from a Gamma distribution depending on the values of $\epsilon$ and $\zeta$, while if approximate DP is desired, the noise vector is drawn from a Gaussian distribution depending on the values of $\epsilon$, $\delta$, and $\zeta$.
We emphasize that for this algorithm, the resulting value $\hat{\boldsymbol\theta}$ is restricted to the set $\mathbb{F}$.

\pkg{DPpack} implements this algorithm using the \code{EmpiricalRiskMinimizationDP.KST} \pkg{R6} class.
Similar to the \code{EmpiricalRiskMinimizationDP.CMS} class, this class provides a general framework for using the algorithm, but is not intended to be used directly.
Child classes inheriting from this class and implementing a specific realization of the ERM for regression algorithm should be used instead.
\pkg{DPpack} implements linear regression in this way, which will be discussed in the subsequent section on regression methods.

\subsection{Logistic Regression}

The two algorithms described in the previous section for privacy preserving ERM for binary classification can be applied to logistic regression.
The loss function given a single observation is the cross entropy loss (or the negative log-likelihood)
\begin{equation}
    \ell_i(\boldsymbol\theta) = -(y_i\log(f_{\boldsymbol\theta}(\mathbf{x}_i)) + (1-y_i)\log(1-f_{\boldsymbol\theta}(\mathbf{x}_i))),
\end{equation}
where $f_{\boldsymbol\theta}(\mathbf{x}_i) = \left(1+e^{-\mathbf{x}_i\boldsymbol\theta}\right)^{-1}$ is the predicted value of $y$.
The regularized objective function given data $\mathcal{D} = (\mathbf{x},\mathbf{y})$ is
\begin{equation}\label{eqn:loss_logsistic}
    \frac{1}{n}\sum_{i=1}^n \left(y_i\log(1+e^{-\mathbf{x}_i\boldsymbol\theta}) + (1-y_i)\log(1+e^{\mathbf{x}_i\boldsymbol\theta})\right) + \frac{\gamma}{n}R(\boldsymbol\theta).
\end{equation}
The loss function in Eqn \eqref{eqn:loss_logsistic} meets all of the regularity conditions necessary for both the output perturbation and the objective perturbation algorithms to satisfy DP\footnote{with $c = 1/4$ for the objective perturbation algorithm \citep{chaudhuri2011}.
Also noted is that differentially private logistic regression was first proved outside of the ERM setting in \citet{Chaudhuri2009}.}. 

\pkg{DPpack} uses the \code{LogisticRegressionDP} \pkg{R6} class to implement differentially private logistic regression in three steps using the \code{EmpiricalRiskMinimizationDP.CMS} framework that is based on the algorithm from \citet{chaudhuri2011}
\footnote{Logistic regression can also be implemented using the algorithm from \citet{Kifer2012} implemented in \code{EmpiricalRiskMinimizationDP.KST}.
While \pkg{DPpack} does not currently provide such an application in its current version, users are able to develop their own using the inheritance structure if desired.}.

The first step is to construct a \code{LogisticRegressionDP} object.
The constructor for this class accepts a callable function \code{regularizer} for the regularizer function, a privacy budget parameter \code{eps}, a regularization constant \code{gamma}, and a string \code{perturbation.method} indicating whether to use the output or the objective perturbation algorithm.
If the argument \code{perturbation.method} is set to \code{`output'}, the output perturbation algorithm is run. 
The user must ensure in this case that the regularizer meets the necessary requirements, namely that it is differentiable and 1-strongly convex.
If \code{perturbation.method} is set to \code{`objective'}, the objective perturbation algorithm is run.
In this case, the user must ensure that the regularizer is doubly differentiable and 1-strongly convex.
One popular regularization function is the $\ell_2$ regularizer $R(\boldsymbol\theta) = \frac{1}{2}\norm{\boldsymbol\theta}_2^2$.
For convenience, this regularization function (and its gradient) can be used by simply setting \code{regularizer} to \code{`l2'}.
An optional callable function \code{regularizer.gr} representing the gradient of the regularizer can also be provided.

After constructing a \code{LogisticRegressionDP} object, the second step is to train the model with a dataset.
To do this, the user should call the \code{\$fit} method of the constructed object.
This method accepts as arguments a sensitive dataset \code{X} and corresponding sensitive labels for each row \code{y}.
It also accepts numeric vectors giving the global or public bounds on the data in each column of \code{X}.
There are several points to note regarding \code{\$fit}.
First, the method assumes that the binary labels provided by \code{y} are either 0 or 1.
Second, both the output and objective perturbation algorithms assume that for each row $\mathbf{x}_i$ of the input dataset we have $\norm{\mathbf{x}_i}_2\le 1$.
Given that this requirement is not met by most practical datasets, to allow for more realistic datasets to train the model, the \code{\$fit} method utilizes the provided upper and lower bounds on the columns of \code{X} to pre-process and scale the values of \code{X} in such a way that this constraint is met.
The privacy-preserving algorithm is then run, producing differentially private coefficients for the scaled dataset.
After the private coefficients are generated, these are then post-processed and un-scaled before being stored as the object attribute \code{\$coeff}, so that the stored coefficients correspond to the original data. 
Because both the pre-processing and the post-processing steps rely solely on the global or public bounds, DP is maintained by the post-processing theorem.

Specifically, \code{X} is pre-processed as follows. 
First, the largest in absolute value of the upper and lower bounds on each column are used to scale each column individually such that the largest value in each column is at most 1 in absolute value. 
Second, each value in \code{X} is divided by $\sqrt{p}$, the square root of the number of predictors of \code{X}. 
These two scalings ensure that each row of \code{X} satisfies the necessary constraints for DP.
After training, the post-processing of the private coefficients is then accomplished by dividing each element of the trained vector by the same value used to scale the corresponding column individually in the pre-processing step, then dividing the entire vector by  $\sqrt{p}$.

The original privacy-preserving ERM algorithms assume there is no bias term present in the predictor function.
If a bias term is necessary, this issue can be partially circumvented by prepending a column of 1s to \code{X} before fitting the model.
In this case, the first element of the fitted vector \code{\$coeff} is essentially the bias term.
The \code{\$fit} method does this when the \code{add.bias} argument is set to \code{TRUE}.
We caution that adding a column of 1s to \code{X} results in an additional column that must be scaled in the pre-processing step, and we recommend not using a bias term if at all possible.

After training the model, the third and final step is to release the trained coefficients or to use them to predict the labels of new datapoints.
The privacy-preserving coefficients are stored in the attribute \code{\$coeff}, which can be directly released without violating privacy guarantees.
Alternatively, the \code{\$predict} method can be used.
This method accepts a set of data \code{X} of the same form (i.e. dimensions, variable order, etc.) as the one provided to the \code{\$fit} method, as well as boolean \code{add.bias} and boolean \code{raw.value} arguments.
The method then returns a matrix of predicted values corresponding to each row of \code{X} based on the logistic regression predictor function $f_{\boldsymbol\theta}$ and the trained and stored coefficients \code{\$coeff}.
The \code{add.bias} argument should be set to the same value as the identically named argument was when the \code{\$fit} method was called.
The \code{raw.value} argument is used to indicate whether the returned matrix should consist of the raw scores from the logistic regression predictor function (i.e. real numbers between 0 and 1), or whether it should consist of predicted labels for the rows (i.e. 0 or 1 values) obtained by rounding the scores.

The following example shows the usage of the \code{LogisticRegressionDP} class on a 2-dimensional toy dataset.
\begin{example}
    # Simulate train dataset X and y, and test dataset Xtest and ytest
    N <- 200
    K <- 2
    X <- data.frame()
    y <- data.frame()
    for (j in (1:K)){
      t <- seq(-.25, .25, length.out = N)
      if (j==1) m <- rnorm(N,-.2, .1)
      if (j==2) m <- rnorm(N, .2, .1)
      Xtemp <- data.frame(x1 = 3*t , x2 = m - t)
      ytemp <- data.frame(matrix(j-1, N, 1))
      X <- rbind(X, Xtemp)
      y <- rbind(y, ytemp)
    }
    # Bounds for X based on construction
    upper.bounds <- c( 1, 1)
    lower.bounds <- c(-1,-1)
    
    # Train-test split
    Xtest <- X[seq(1,(N*K),10),]
    ytest <- y[seq(1,(N*K),10),,drop=FALSE]
    X <- X[-seq(1,(N*K),10),]
    y <- y[-seq(1,(N*K),10),,drop=FALSE]
    
    # Construct object for logistic regression
    regularizer <- function(coeff) coeff%*%coeff/2
    regularizer.gr <- function(coeff) coeff
    eps <- 1
    gamma <- 0.1
    lrdp <- LogisticRegressionDP$new(regularizer, eps, gamma,
                                     regularizer.gr = regularizer.gr)
    
    # Fit with data
    lrdp$fit(X, y, upper.bounds, lower.bounds) # No bias term
    lrdp$coeff # Gets private coefficients
    #> 1.449110 5.562798
    
    # Predict new data points
    predicted.y <- lrdp$predict(Xtest)
    n.errors <- sum(predicted.y!=ytest)
\end{example}

\subsection{Support Vector Machine (SVM)}

The privacy-preserving binary classification ERM algorithms can also be applied to linear and nonlinear SVM. 
For notational simplicity, we let $\{-1, 1\}$  be the binary labels for $y$ when defining loss functions in SVM; for the implementation, for consistency with the \code{LogisticRegressionDP} class, we require $y$ in the input dataset to be coded in $\{0, 1\}$.

For linear SVM, the loss function given a single observation is the hinge loss
\begin{equation}
    \ell_i(\boldsymbol{\theta}) = \max(0, 1-y_if_{\boldsymbol\theta}(\mathbf{x}_i)),
\end{equation}
where $f_{\boldsymbol\theta}(\mathbf{x}_i) = \mathbf{x}_i\boldsymbol\theta$ is the predicted value of $y$.
The regularized objective function given data $\mathcal{D} = (\mathbf{x}, \mathbf{y})$ is
\begin{equation}
    \frac{1}{n}\sum_{i=1}^n \max(0,1-y_i\mathbf{x}_i\boldsymbol\theta) + \frac{\gamma}{n}R(\boldsymbol\theta).
\end{equation}
Unfortunately, the hinge loss is not differentiable everywhere and therefore does not satisfy the requirements for privacy-preserving ERM.
One solution to this (used by \pkg{DPpack}) is to use the smooth Huber loss approximation to the hinge loss \citep{Chapelle2007} defined by
\begin{equation}\label{eqn:hinge}
    \ell_\textrm{Huber}(z) = \begin{cases} 
      0, & \textrm{if} \;\; z > 1 + h \\
      \frac{1}{4h}(1+h-z)^2, & \textrm{if} \;\; |1-z| \le h \\
      1 - z, & \textrm{if} \;\; z < 1 - h
   \end{cases}
\end{equation}
for a given Huber loss parameter $h$.
Figure \ref{fig:huber_loss} shows a comparison between the Huber loss and the hinge loss for various values of $h$.
For linear SVM, the described predictor function and the Huber loss meet all of the requirements necessary for both the output perturbation and the objective perturbation algorithms for DP\footnote{with $c = 1/2h$ for the objective perturbation algorithm \citep{chaudhuri2011}}.

\begin{figure}[!htb]
    \centering
    \includegraphics[width=.5\textwidth]{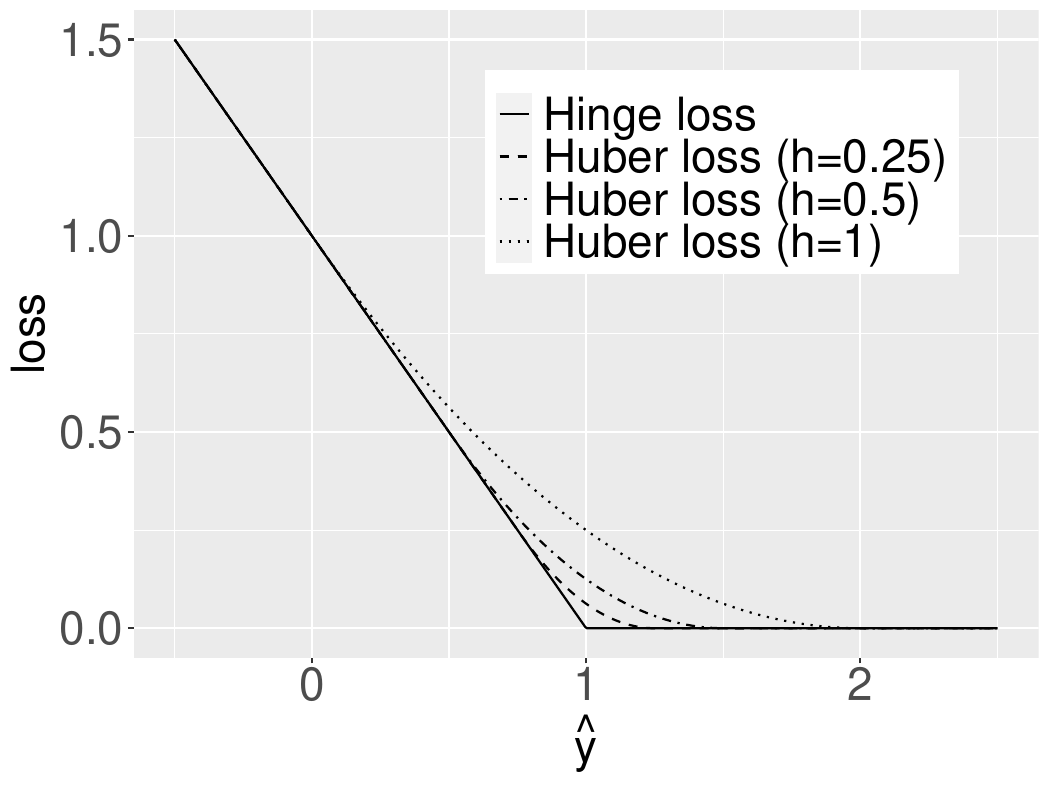}
    \caption{Comparison between Huber and hinge loss assuming $y=1$.}
    \label{fig:huber_loss}
\end{figure}

Linear SVM implicitly assumes that the given dataset is (at least approximately) linearly separable.
When this is not the case, nonlinear SVM is a better choice. 
Intuitively, nonlinear SVM first maps the potentially linearly non-separable input data in the original space to a higher dimension in which the data is linearly separable, then uses linear SVM in the higher-dimensional space. Performing this mapping directly suffers from the curse of dimensionality, and computations on the higher-dimensional dataset quickly become prohibitively expensive. 
For that reason, the kernel trick is used so that SVM can be applied easily in practice.
Briefly, rather than explicitly transforming the data to the higher-dimensional space, the kernel trick utilizes a kernel function $k$ to produce a similarity score between two datapoints in the original dimension.
This is more computationally efficient than performing computations in the higher-dimensional space.
One popular kernel function is the Gaussian or radial kernel
\begin{equation}
    k(\mathbf{x}, \mathbf{x}^\prime) = \exp\big(-\beta|\!|\mathbf{x}-\mathbf{x}^\prime|\!|_2^2\big),
    \label{eq:Gaussian_kernel}
\end{equation}
where $\beta$ is a Gaussian kernel hyperparameter and equals to $p^{-1}$ by default. 
The optimized predictor function at $\mathbf{x}$ is a linear combination of kernel functions
\begin{equation}\label{eqn:y_kernel}
    \hat{y}=f(\mathbf{x}) = \textstyle\sum_{i=1}^na_ik(\mathbf{x}_i, \mathbf{x}),
\end{equation}
where $\mathbf{x}_i$ is the input of the original dataset.

When there is no privacy concern, one may release estimated $a_i$ and the observed input $\mathbf{x}_i$, which can be plugged in Eqn \eqref{eqn:y_kernel} to predict the label of a given data point $\mathbf{x}$.
For privacy-preserving analysis, this practice poses problems due to the direct release of $\mathbf{x}_i$.
\citet{chaudhuri2011} avoids this issue by using random projections to approximate the desired kernel function.
Specifically, the algorithm first randomly samples $D$ vectors $\mathbf{z}_j$ based on the desired kernel function according to the approximation technique from \citet{Rahimi2007, Rahimi2008}.
The algorithm then produces $D$-dimensional data $\mathbf{v}_i$ (using $\mathbf{z}_j$) for each $\mathbf{x}_i$, representing an approximate projection of each of the original $\mathbf{x}_i$ onto the kernel space.
Finally, the differentially private ERM algorithm for linear SVM is run on the new dataset $(\mathbf{v}_i,y_i)$.
The vectors $\mathbf{z}_j$ are not functions of the observed dataset, meaning the privacy-preserving linear SVM algorithm satisfies $\epsilon$- DP for a given $\epsilon$.
Therefore, this algorithm also satisfies $\epsilon$-DP when releasing the sampled vectors $\mathbf{z}_j$ and the estimated coefficients from the linear SVM.

\pkg{DPpack} implements differentially private SVM via the \code{svmDP} \pkg{R6} class using the framework provided by \code{EmpiricalRiskMinimizationDP.CMS}.
Like the logistic regression model, using the SVM model requires three steps.
The first step is to construct an \code{svmDP} object.
The constructor for this class accepts a callable function \code{regularizer} for the regularizer function, a privacy budget parameter \code{eps}, a regularization constant \code{gamma}, a string \code{perturbation.method} indicating whether to use the output or the objective perturbation algorithm, a string \code{kernel} for the kernel used in SVM, and a constant \code{huber.h} defining the $h$ value in the Huber loss in Eqn \eqref{eqn:hinge}. 
Setting \code{regularizer} to \code{`l2'} uses the $\ell_2$ regularization function and its gradient.
The \code{perturbation.method} argument operates identically to the argument of the same name used in constructing a \code{LogisticRegressionDP} object, and expects the user to verify the same requirements for the \code{regularizer} and \code{regularizer.gr} functions.
The \code{kernel} argument can be set to either \code{`linear'} or \code{`Gaussian'}. 
In the former case, linear SVM is run using the specified predictor function and the Huber loss; for the latter, the Gaussian kernel approximation algorithm is run, where the constructor also requires the specification of two additional arguments: \code{D} to indicate the dimensionality of the projection dataset $\mathbf{v}_i$ and \code{kernel.param} to indicate the value of $\beta$ in Eqn \eqref{eq:Gaussian_kernel}.

After constructing the \code{svmDP} object, the second step is to train the model on a dataset. 
Users should call the \code{\$fit} method of the constructed object. 
This method accepts input data \code{X} with labels \code{y}, numeric vector global or public bounds for each column of \code{X}, and a boolean \code{add.bias} indicating whether to add a column of 1s to \code{X} to act as a bias variable\footnote{The \code{add.bias} argument functions analogously to the respective argument for the \code{LogisticRegressionDP} class, and we again recommend not using a bias term if at all possible.}.
If \code{kernel} is set to \code{`linear'} when the object is constructed in the first step, the method finds $\hat{\boldsymbol\theta}$ satisfying \code{eps}-DP, where \code{eps} is the privacy budget provided when the object is initialized.
If \code{kernel} is set to \code{`Gaussian'}, the method first converts \code{X} to the $D$-dimensional new dataset \code{V}, then finds $\hat{\boldsymbol\theta}$ corresponding to \code{V} satisfying \code{eps}-DP.

For linear SVM, the same pre-processing of \code{X} using the provided bounds on its columns and subsequent post-processing of the private coefficients is performed. 
The results are again stored in the \code{\$coeff} attribute.
For Gaussian kernel nonlinear SVM, the mapping from \code{X} to \code{V} ensures that each row $\mathbf{v}_i$ satisfies $\norm{\mathbf{v}_i}_2\le1$ regardless of the values of $\mathbf{x}_i$.
For this reason, no pre-processing of \code{X} is needed. 
In fact, providing bounds on the columns of \code{X} when calling the \code{\$fit} method is unnecessary for the Gaussian kernel approximation.

The third and final step is to release the estimated coefficients \code{\$coeff}\footnote{For Gaussian kernel SVM, the dimension conversion function, \code{\$XtoV}, can also be released in conjunction with \code{\$coeff}.} or use them to predict the labels given a set of new datapoints. 
For the latter, the \code{\$predict} method can be used, which accepts input \code{X} of the same form (i.e. dimensions, variable order, etc.) as the one provided to the \code{\$fit} method, as well as boolean \code{add.bias} and boolean \code{raw.value} arguments\footnote{The \code{add.bias} and \code{raw.value} arguments operate analogously to the respective arguments for \code{LogisticRegressionDP}.}. 
It returns a matrix of predicted values corresponding to each row of \code{X}.

The following example shows how to use the \code{svmDP} class.
\begin{example}
    # Simulate training dataset X and y, and testing dataset Xtest and ytest
    N <- 400
    X <- data.frame()
    y <- data.frame()
    for (i in (1:N)){
      Xtemp <- data.frame(x1 = rnorm(1,sd=.28) , x2 = rnorm(1,sd=.28))
      if (sum(Xtemp^2)<.15) ytemp <- data.frame(y=0)
      else ytemp <- data.frame(y=1)
      X <- rbind(X, Xtemp)
      y <- rbind(y, ytemp)
    }
    
    # Train-test split
    Xtest <- X[seq(1,N,10),]
    ytest <- y[seq(1,N,10),,drop=FALSE]
    X <- X[-seq(1,N,10),]
    y <- y[-seq(1,N,10),,drop=FALSE]
    
    # Construct object for SVM
    regularizer <- `l2'
    eps <- 1
    gamma <- 0.1
    kernel <- `Gaussian'
    D <- 20
    svmdp <- svmDP$new(regularizer, eps, gamma, kernel=kernel, D=D)
    
    # Fit with data (note no bounds necessary because kernel=`Gaussian')
    svmdp$fit(X, y) # No bias term
    
    # Predict new data points
    predicted.y <- svmdp$predict(Xtest)
    n.errors <- sum(predicted.y!=ytest)
\end{example}

\subsection{Outcome Weighted Learning}

Outcome weighted learning (OWL) \citep{Zhao2012} is a technique used for determining individualized treatment rules (ITRs), and can be categorized broadly as a method for causal inference ML.
The primary goal of ITR is to derive a treatment assignment function that maps an individual's set of characteristics to a treatment that maximizes the expected benefit to that individual.
A significant strength of OWL is its ability to tailor treatment assignments in response to individual characteristics, rather than using a one-size-fits-all approach.
Its development was motivated by developing techniques for precision medicine \citep{Council2011toward,Collins2015new} in randomized clinical trials, though other potential applications include personalized advertising \citep{Wang2015Robust, Sun2015}, and recommender systems \citep{Schnabel2016, Lada2019}.
The original ITR problem considered in \citet{Zhao2012} is to find the treatment assignment function $T$ by maximizing the expected treatment benefit $E[\frac{B}{P(A|\mathbf{x})}\mathbbm{1}(A=T(\mathbf{x}))]$, where $A$ and $B$ are random variables representing the randomly assigned treatment and observed benefit, respectively, $P$ is the conditional probability function, and $\mathbbm{1}$ is the indicator function.

The key insight in \citet{Zhao2012} that produces the OWL framework is that the expected benefit problem can be reformulated as the weighted SVM problem
\begin{equation}
    \frac{1}{n}\sum_{i=1}^n \frac{B_i}{P(A_i|\mathbf{x}_i)}\max(0,1-A_i\mathbf{x}_i\boldsymbol\theta) + \frac{\gamma}{n}\norm{\boldsymbol\theta}_2,
\end{equation}
where $\mathcal{D} = (\mathbf{x}, \mathbf{A}, \mathbf{B})$.
It is straightforward to see that this is a generalization of the standard SVM case to the case where individual observations are unevenly weighted according to the weights $w_i = \frac{B_i}{P(A_i|\mathbf{x}_i)}$.
\citet{Giddens2023} showed that weighted ERM in general can be made to satisfy $\epsilon$-DP via output perturbation, as long as a global bound on the weights is provided.
In order to incorporate OWL into \pkg{DPpack}, we generally implement DP weighted ERM through the general \code{WeightedERMDP.CMS} class.
The \code{svmDP} class described in the previous section inherits from \code{WeightedERMDP.CMS}, which permit users to provide unequal observation weights such as those found in OWL.

The following example shows how to use the \code{svmDP} class with weighted observations.
\begin{example}
    # Simulate train dataset X and y, and test dataset Xtest and ytest
    N <- 200
    K <- 2
    X <- data.frame()
    y <- data.frame()
    for (j in (1:K)){
      t <- seq(-.25, .25, length.out = N)
      if (j==1) m <- rnorm(N,-.2, .1)
      if (j==2) m <- rnorm(N, .2, .1)
      Xtemp <- data.frame(x1 = 3*t , x2 = m - t)
      ytemp <- data.frame(matrix(j-1, N, 1))
      X <- rbind(X, Xtemp)
      y <- rbind(y, ytemp)
    }
    # Bounds for X based on construction
    upper.bounds <- c( 1, 1)
    lower.bounds <- c(-1,-1)
    
    # Train-test split
    Xtest <- X[seq(1,(N*K),10),]
    ytest <- y[seq(1,(N*K),10),,drop=FALSE]
    X <- X[-seq(1,(N*K),10),]
    y <- y[-seq(1,(N*K),10),,drop=FALSE]

    # Weights
    weights <- rep(1, nrow(y)) # Uniform weighting
    weights[nrow(y)] <- 0.5 # Half weight for last observation
    wub <- 1 # Upper bound on weights
    
    # Construct object for logistic regression
    regularizer <- function(coeff) coeff%*%coeff/2
    regularizer.gr <- function(coeff) coeff
    eps <- 1
    gamma <- 0.1
    perturbation.method <- `output'
    svmdp <- svmDP$new(regularizer, eps, gamma, perturbation.method,
                      regularizer.gr = regularizer.gr)
    
    # Fit with data
    svmdp$fit(X, y, upper.bounds, lower.bounds, weights=weights,
              weights.upper.bound=wub)
    svmdp$coeff # Gets private coefficients
    #> 1.547518 13.456029
    
    # Predict new data points
    predicted.y <- svmdp$predict(Xtest)
    n.errors <- sum(predicted.y!=ytest)
\end{example}

\subsection{Linear Regression}

The differentially private ERM algorithm for regression problems can be applied to linear regression.
The loss function given a single observation is the squared error
\begin{equation}
    \ell_i(\boldsymbol\theta) = \frac{(f_{\boldsymbol\theta}(\mathbf{x}_i)-y_i)^2}{2},
\end{equation}
where the $f_{\boldsymbol\theta}(\mathbf{x}_i) = \mathbf{x}_i\boldsymbol\theta$ is the predicted value of $y$. 
The regularized objective function given data $\mathcal{D} = (\mathbf{x}, \mathbf{y})$ is
\begin{equation}
    \frac{1}{n}\sum_{i=1}^n \frac{(\mathbf{x}_i\boldsymbol\theta - y_i)^2}{2} + \frac{\gamma}{n}R(\boldsymbol\theta).
\end{equation}
In order to satisfy all of the assumptions needed to ensure privacy for the ERM algorithm, we must assume that each $\mathbf{x}_i$, as well as the coefficient vector $\boldsymbol\theta$ have a bounded $\ell_2$ norm. 
For the purposes of \pkg{DPpack}, we choose to bound $\mathbf{x}_i$ by $\norm{\mathbf{x}_i}_2\le\sqrt{p}$ and the coefficient vector by $\norm{\boldsymbol\theta}_2\le\sqrt{p}$, where $p$ is the number of predictors, following \citet{Kifer2012}.
This implies that each value of the output $y$ is contained in $[-p, p]$ automatically.
With these assumptions, the conditions for the differentially private regression ERM algorithm are satisfied for linear regression with parameters $\mathbb{F} = \{\boldsymbol\theta\in \mathbb{R}^p : \norm{\boldsymbol\theta}_2 \le \sqrt{p}\}$, $\zeta = 2p^{3/2}$, and $\lambda = p$.

\pkg{DPpack} implements differentially private linear regression via the \code{LinearRegressionDP} \pkg{R6} class using the framework of \code{EmpiricalRiskMinimizationDP.KST}.
Similar to the classification models, this is done in three steps: constructing a \code{LinearRegressionDP} object, training the model by calling the \code{\$fit} method of the constructed object, and releasing the trained coefficients \code{\$coeff} or using them for prediction via \code{\$predict}.
The arguments and specification of the construction and prediction steps are similar to the those for \code{LogisticRegressionDP} and \code{svmDP}, so we refer the reader to those sections for explanations of the arguments.

There are a few minor differences in the training step via the \code{\$fit} method when compared to \code{LogisticRegressionDP} and \code{svmDP}.
First, the arguments \code{lower.bounds} and \code{upper.bounds} should be vectors representing the global or public bounds on both the columns of \code{X} and the values of \code{y}. 
If \code{X} has $n$ columns, then each vector of bounds should be of length $n+1$. 
The first $n$ elements of the vectors correspond to the bounds on the $n$ columns of \code{X}, and are in the same order as the respective columns.
The last element of the vectors corresponds to the bounds on the values in \code{y}.
Similar to the training step for \code{LogisticRegressionDP} and \code{svmDP}, these bounds are used to pre-process \code{X} and \code{y} so that they satisfy the necessary constraints for privacy. 
The pre-processing/post-processing is essentially the same for \code{LinearRegressionDP} as it is for the classification methods, except that \code{y} is also shifted (and the resulting coefficients unshifted) to be centered at $0$ if \code{add.bias} is set to \code{TRUE}.

The following example shows how to use the \code{LinearRegressionDP} class.
\begin{example}
    # Simulate an example dataset
    n <- 500
    X <- data.frame(X=seq(-1,1,length.out = n))
    true.theta <- c(-.3,.5) # First element is bias term
    p <- length(true.theta)
    y <- true.theta[1] + as.matrix(X)%*%true.theta[2:p] + rnorm(n=n,sd=.1)
    # Bounds based on construction. We assume y has values between -p and p
    upper.bounds <- c(1, p) # Bounds for X and y
    lower.bounds <- c(-1, -p) # Bounds for X and y
    
    # Construct object for linear regression
    regularizer <- 'l2'
    eps <- 1
    delta <- 0.01 # Indicates to use approximate (1,0.01)-DP
    gamma <- 1
    lrdp <- LinearRegressionDP$new('l2', eps, delta, gamma)
    
    # Fit with data
    lrdp$fit(X, y, upper.bounds, lower.bounds, add.bias=TRUE)
    lrdp$coeff # Gets private coefficients
    #> -0.3812353  0.3704237
    
    # Predict new data points
    Xtest <- data.frame(X=c(-.5, -.25, .1, .4))
    predicted.y <- lrdp$predict(Xtest, add.bias=TRUE)
\end{example}

\subsection{Hyperparameter Tuning}

Model training often involves the selection of hyperparameter values such as, for example, the constant $\gamma$ for the regularizer in Eqn \eqref{eq:regularized_ERM} or \eqref{eq:private_regularized_ERM}. 
Poorly selected values for these hyperparameters can result in models with poor performance. 
Often, hyperparameter selection relies on the observed dataset itself, resulting in privacy costs in the setting of privacy-preserving analysis. \citet{chaudhuri2011} presents an algorithm for privacy-preserving hyperparameter tuning based on the exponential mechanism, which is implemented in \pkg{DPpack}.

For binary classification models, differentially private hyperparameter tuning is realized in \pkg{DPpack} via the \code{tune\_classification\_model} function.
It accepts as inputs a list of model objects \code{models} of the same type\footnote{Such as a list of \code{LogisticRegressionDP} objects. Each model object must have the same privacy budget parameters.}, each constructed with a different value from the set of potential hyperparameter values, observed input \code{X}, labels \code{y}, vectors representing global or public bounds on the columns of \code{X}, and a boolean \code{add.bias} argument.
The function splits \code{X} and \code{y} into $m+1$ equally sized sub-datasets, where $m$ is the number of candidate models, and trains each model on one of the sub-datasets.
The negative of the misclassification frequency by each model on the labels of the final sub-dataset is used as the utility function $u$ for the exponential mechanism.
It can be easily seen that the $\ell_1$-global sensitivity of $u$ is $\Delta_{1,u} = 1$. 
The exponential mechanism is used to select and return one of the trained models provided with $\epsilon$-DP.

For example, assume one wishes to select a constant for the $l_2$ regularizer from the set $\{100, 1, 0.0001\}$ for privacy-preserving logistic regression.
To do this, three objects from the \code{LogisticRegressionDP} class are constructed with the same privacy budget parameter \code{eps} and initialized with one of the three constant values.
The three model objects are then passed into the tuning function, and the exponential mechanism returns one of them. 
The remaining arguments for the tuning function, \code{X}, \code{y}, \code{upper.bounds}, \code{lower.bounds}, and \code{add.bias}, should be given values according to their respective descriptions in the \code{\$fit} method of the corresponding \pkg{R6} class being used. 
An example of this situation follows.
\begin{example}
    # Simulate a training dataset (X, y), and testing dataset (Xtest, ytest)
    N <- 200
    K <- 2
    X <- data.frame()
    y <- data.frame()
    for (j in (1:K)){
      t <- seq(-.25,.25,length.out = N)
      if (j==1) m <- rnorm(N,-.2,.1)
      if (j==2) m <- rnorm(N, .2,.1)
      Xtemp <- data.frame(x1 = 3*t , x2 = m - t)
      ytemp <- data.frame(matrix(j-1, N, 1))
      X <- rbind(X, Xtemp)
      y <- rbind(y, ytemp)
    }
    # Bounds for X based on construction
    upper.bounds <- c( 1, 1)
    lower.bounds <- c(-1,-1)
    
    # Train-test split
    Xtest <- X[seq(1,(N*K),10),]
    ytest <- y[seq(1,(N*K),10),,drop=FALSE]
    X <- X[-seq(1,(N*K),10),]
    y <- y[-seq(1,(N*K),10),,drop=FALSE]
    y <- as.matrix(y)
    
    # Grid of gamma values for tuning logistic regression model
    grid.search <- c(100, 1, .0001)
    # Construct objects for logistic regression parameter tuning
    eps <- 1 # Privacy budget should be the same for all models
    lrdp1 <- LogisticRegressionDP$new("l2", eps, grid.search[1])
    lrdp2 <- LogisticRegressionDP$new("l2", eps, grid.search[2])
    lrdp3 <- LogisticRegressionDP$new("l2", eps, grid.search[3])
    models <- c(lrdp1, lrdp2, lrdp3)
    
    # Tune using data and bounds for X based on its construction
    tuned.model <- tune_classification_model(models, X, y, 
                                             upper.bounds, lower.bounds)
    tuned.model$gamma # Gives resulting selected hyperparameter
    #> 0.0001
   
    # tuned.model can be used in the same way as any 
    # LogisticRegressionDP model
    predicted.y <- tuned.model$predict(Xtest)
    n.errors <- sum(predicted.y!=ytest)
\end{example}

\pkg{DPpack} also implements differentially private hyperparameter tuning for linear regression via the \code{tune\_linear\_regression\_model} function.
This function was inspired by the binary classification hyperparameter tuning algorithm from \citet{chaudhuri2011} as well as the feature selection algorithm for high-dimensional regression from \citet{Kifer2012}.
This function accepts the same input arguments as the \code{tune\_classification\_model} function, except that the \code{models} argument should be a list of constructed \code{LinearRegressionDP} objects with the same privacy budget parameters \code{eps} and \code{delta}.
The function then splits the provided data \code{X} and \code{y} into $m+1$ equally sized sub-datasets, where $m$ is the number of provided models, and trains each model on one of the sub-datasets.
The negative of the square of the Euclidean distance between the predicted values and the true values for the remaining sub-dataset is defined to be the utility function $u$ for each of the models, the $\ell_1$-global sensitivity for which is given in Theorem \ref{thm:l1_sensitivity_u}. 
Finally, the exponential mechanism is used to select and return one of the trained models provided with $($\code{eps}, \code{delta}$)$-DP.

\begin{theorem}\label{thm:l1_sensitivity_u}
    Let $c_0$ and $c_1$ be the global or public lower and upper bounds, respectively, on the possible values of $y_i$. 
    Let $g$ be the linear regression model with coefficient parameters $\boldsymbol\theta$. 
    For a dataset $\mathcal{D} \!=\!(\mathbf{x}_i,y_i)$ with $n$ rows, define $u(\mathcal{D}, g)\!=\! -\sum_{i=1}^n(g(x_i)-y_i)^2$. 
    The $\ell_1$-global sensitivity of $u$ is given by 
    \begin{equation}
        \Delta_{1,u} = (c_1-c_0)^2.
    \end{equation}
\end{theorem}
\begin{proof}
    Let $\mathcal{D}_1$ and $\mathcal{D}_2$ be (bounded) neighboring datasets.
    Without loss of generality, assume they differ only in their first element and define $(x_1, y_1) \in \mathcal{D}_1$ and $(x_1^\prime,y_1^\prime) \in \mathcal{D}_2$.
    Then
    \begin{align*}
        \Delta_{1,u} &= \max_{g} \max_{\mathcal{D}_1, \mathcal{D}_2} |u(\mathcal{D}_1, g) - u(\mathcal{D}_2, g)| \\
        &= \max_{g} \max_{\mathcal{D}_1, \mathcal{D}_2} |(g(x_1)-y_1)^2 - (g(x_1^\prime) - y_1^\prime)^2|.
    \end{align*}
    Given that $(g(x_1)-y_1)^2\ge0$ and $(g(x'_1)-y'_1)^2\ge0$ for all $g$ and for all $\mathcal{D}_1$ and $\mathcal{D}_2$ with $d(\mathcal{D}_1, \mathcal{D}_2)=1$, 
    \begin{align*}
        \Delta_{1,u} &\le \max_{x_1, y_1} (g(x_1) - y_1)^2 = \max_{x_1, y_1} (x_1\boldsymbol\theta - y_1)^2 = (c_1 - c_0)^2,
    \end{align*}
    where we note the last step is a result of the assumptions made on the bounds of $\norm{x_1}_2$, $\norm{\boldsymbol\theta}_2$, and $|y_1|$ in order to ensure DP for linear regression.
    For unbounded DP, $\Delta_{1,u}=\max_{g} \max_{\mathcal{D}_1, \mathcal{D}_2} (g(x_1)-y_1)^2=(c_1 - c_0)^2$, the same as in the bounded case.
\end{proof}

Similar to the \code{tune\_classification\_model} function, the list of models provided to \code{tune\_linear\_regression\_model} should be a list of objects constructed using the \pkg{R6} class \code{LinearRegressionDP} with a different hyperparameter value and the same privacy budget parameters provided to each model. 
The remaining arguments for the tuning function, \code{X}, \code{y}, \code{upper.bounds}, \code{lower.bounds}, and \code{add.bias}, should be given values according to their respective descriptions in the \code{\$fit} method of the \code{LinearRegressionDP} class. 
An example of using the tuning function for the regularization constant for linear regression follows.
\begin{example}
    # Simulate an example dataset
    n <- 500
    X <- data.frame(X=seq(-1,1,length.out = n))
    true.theta <- c(-.3,.5) # First element is bias term
    p <- length(true.theta)
    y <- true.theta[1] + as.matrix(X)%*%true.theta[2:p] + rnorm(n=n,sd=.1)
    # Bounds for X and y based on their construction
    upper.bounds <- c( 1, 2) # Bounds for X and y
    lower.bounds <- c(-1,-2) # Bounds for X and y
    
    # Grid of possible gamma values for tuning linear regression model
    grid.search <- c(100, 1, .0001)
    # Construct objects for logistic regression parameter tuning
    # Privacy budget should be the same for all models
    eps <- 1
    delta <- 0.01
    linrdp1 <- LinearRegressionDP$new("l2", eps, delta, grid.search[1])
    linrdp2 <- LinearRegressionDP$new("l2", eps, delta, grid.search[2])
    linrdp3 <- LinearRegressionDP$new("l2", eps, delta, grid.search[3])
    models <- c(linrdp1, linrdp2, linrdp3)
    tuned.model <- tune_linear_regression_model(models, X, y, upper.bounds,
                                                lower.bounds, add.bias=TRUE)
    tuned.model$gamma # Gives resulting selected hyperparameter
    #> 100
    
    # tuned.model result can be used the same as a trained 
    # LogisticRegressionDP model
    tuned.model$coeff # Gives coefficients for tuned model
    #> -0.5038190  0.2589978
    
    # Simulate a test dataset for prediction
    Xtest <- data.frame(X=c(-.5, -.25, .1, .4))
    predicted.y <- tuned.model$predict(Xtest, add.bias=TRUE)
\end{example}









\bibliography{references}

\begin{thebibliography}{62}
\providecommand{\natexlab}[1]{#1}
\providecommand{\url}[1]{\texttt{#1}}
\expandafter\ifx\csname urlstyle\endcsname\relax
  \providecommand{\doi}[1]{doi: #1}\else
  \providecommand{\doi}{doi: \begingroup \urlstyle{rm}\Url}\fi

\bibitem[Abadi et~al.(2016)Abadi, Chu, Goodfellow, McMahan, Mironov, Talwar, and Zhang]{Abadi2016}
Martin Abadi, Andy Chu, Ian Goodfellow, H.~Brendan McMahan, Ilya Mironov, Kunal Talwar, and Li~Zhang.
\newblock Deep learning with differential privacy.
\newblock In \emph{Proceedings of the 2016 ACM SIGSAC Conference on Computer and Communications Security}, CCS '16, page 308–318, New York, NY, USA, 2016. Association for Computing Machinery.
\newblock ISBN 9781450341394.
\newblock \doi{10.1145/2976749.2978318}.
\newblock URL \url{https://doi.org/10.1145/2976749.2978318}.

\bibitem[Ahn(2015)]{Ahn2015}
Sejin Ahn.
\newblock Whose genome is it anyway?: Re-identification and privacy protection in public and participatory genomics.
\newblock \emph{The San Diego law review}, 52:\penalty0 751, 2015.

\bibitem[Apple(2017)]{AppleDP}
Apple.
\newblock Learning with privacy at scale.
\newblock \url{https://machinelearning.apple.com/research/learning-with-privacy-at-scale}, December 2017.
\newblock Accessed: August 27, 2023.

\bibitem[Barrientos et~al.(2019)Barrientos, Reiter, Machanavajjhala, and Chen]{Barrientos2019}
Andrés~F. Barrientos, Jerome~P. Reiter, Ashwin Machanavajjhala, and Yan Chen.
\newblock Differentially private significance tests for regression coefficients.
\newblock \emph{Journal of Computational and Graphical Statistics}, 28\penalty0 (2):\penalty0 440--453, 2019.
\newblock \doi{10.1080/10618600.2018.1538881}.

\bibitem[Bowen and Liu(2020)]{bowen2020comparative}
Claire~McKay Bowen and Fang Liu.
\newblock Comparative study of differentially private data synthesis methods.
\newblock 2020.

\bibitem[Bu et~al.(2020)Bu, Dong, Long, and Su]{Bu2020}
Zhiqi Bu, Jinshuo Dong, Qi~Long, and Weijie Su.
\newblock Deep learning with gaussian differential privacy.
\newblock \emph{Harvard Data Science Review}, 2\penalty0 (3), September 2020.
\newblock \doi{10.1162/99608f92.cfc5dd25}.
\newblock URL \url{https://hdsr.mitpress.mit.edu/pub/u24wj42y}.

\bibitem[Bureau(2021)]{USCensusDP}
United States~Census Bureau.
\newblock 2020 decennial census: Processing the count: Disclosure avoidance modernization.
\newblock \url{https://www.census.gov/programs-surveys/decennial-census/decade/2020/planning-management/process/disclosure-avoidance.html}, November 2021.
\newblock Accessed: August 27, 2023.

\bibitem[Chapelle(2007)]{Chapelle2007}
Olivier Chapelle.
\newblock Training a support vector machine in the primal.
\newblock \emph{Neural Computation}, 19\penalty0 (5):\penalty0 1155--1178, 2007.
\newblock \doi{10.1162/neco.2007.19.5.1155}.

\bibitem[Chaudhuri and Monteleoni(2009)]{Chaudhuri2009}
Kamalika Chaudhuri and Claire Monteleoni.
\newblock Privacy-preserving logistic regression.
\newblock In D~Koller, D~Schuurmans, Y~Bengio, and L~Bottou, editors, \emph{Advances in Neural Information Processing Systems}, volume~21. Curran Associates, Inc., 2009.
\newblock URL \url{https://proceedings.neurips.cc/paper/2008/file/8065d07da4a77621450aa84fee5656d9-Paper.pdf}.

\bibitem[Chaudhuri et~al.(2011)Chaudhuri, Monteleoni, and Sarwate]{chaudhuri2011}
Kamalika Chaudhuri, Claire Monteleoni, and Anand~D. Sarwate.
\newblock Differentially private empirical risk minimization.
\newblock \emph{Journal of Machine Learning Research}, 12\penalty0 (29):\penalty0 1069--1109, 2011.
\newblock URL \url{http://jmlr.org/papers/v12/chaudhuri11a.html}.

\bibitem[Chaudhuri et~al.(2013)Chaudhuri, Sarwate, and Sinha]{Chaudhuri2013}
Kamalika Chaudhuri, Anand~D. Sarwate, and Kaushik Sinha.
\newblock A near-optimal algorithm for differentially-private principal components.
\newblock \emph{Journal of Machine Learning Research}, 14\penalty0 (53):\penalty0 2905--2943, 2013.
\newblock URL \url{http://jmlr.org/papers/v14/chaudhuri13a.html}.

\bibitem[Chen et~al.(2016)Chen, Machanavajjhala, Reiter, and Barrientos]{Chen2016}
Yan Chen, Ashwin Machanavajjhala, Jerome~P. Reiter, and Andrés~F. Barrientos.
\newblock Differentially private regression diagnostics.
\newblock In \emph{2016 IEEE 16th International Conference on Data Mining (ICDM)}, pages 81--90, 2016.
\newblock \doi{10.1109/ICDM.2016.0019}.

\bibitem[Collins and Varmus(2015)]{Collins2015new}
Francis~S Collins and Harold Varmus.
\newblock A new initiative on precision medicine.
\newblock \emph{New England journal of medicine}, 372\penalty0 (9):\penalty0 793--795, 2015.

\bibitem[Couch et~al.(2019)Couch, Kazan, Shi, Bray, and Groce]{Couch2019}
Simon Couch, Zeki Kazan, Kaiyan Shi, Andrew Bray, and Adam Groce.
\newblock Differentially private nonparametric hypothesis testing.
\newblock In \emph{Proceedings of the 2019 ACM SIGSAC Conference on Computer and Communications Security}, page 737–751, New York, NY, USA, 2019. Association for Computing Machinery.
\newblock ISBN 9781450367479.
\newblock \doi{10.1145/3319535.3339821}.

\bibitem[Council et~al.(2011)]{Council2011toward}
National~Research Council et~al.
\newblock Toward precision medicine: building a knowledge network for biomedical research and a new taxonomy of disease.
\newblock 2011.

\bibitem[Dimitrakakis et~al.(2014)Dimitrakakis, Nelson, Mitrokotsa, and Rubinstein]{Dimitrakakis2014}
Christos Dimitrakakis, Blaine Nelson, Aikaterini Mitrokotsa, and Benjamin I.~P. Rubinstein.
\newblock Robust and private bayesian inference.
\newblock In Peter Auer, Alexander Clark, Thomas Zeugmann, and Sandra Zilles, editors, \emph{Algorithmic Learning Theory}, pages 291--305. Springer International Publishing, 2014.
\newblock ISBN 978-3-319-11662-4.

\bibitem[Ding et~al.(2018)Ding, Nori, Li, and Allen]{Ding2018}
Bolin Ding, Harsha Nori, Paul Li, and Joshua Allen.
\newblock Comparing population means under local differential privacy: With significance and power.
\newblock \emph{Proceedings of the AAAI Conference on Artificial Intelligence}, 32\penalty0 (1), Apr. 2018.
\newblock URL \url{https://ojs.aaai.org/index.php/AAAI/article/view/11301}.

\bibitem[Dwork and Roth(2014)]{DPtextbook}
Cynthia Dwork and Aaron Roth.
\newblock The algorithmic foundations of differential privacy.
\newblock \emph{Foundations and Trends in Theoretical Computer Science}, 9\penalty0 (3–4):\penalty0 211–407, aug 2014.
\newblock ISSN 1551-305X.
\newblock URL \url{https://doi.org/10.1561/0400000042}.

\bibitem[Dwork et~al.(2006{\natexlab{a}})Dwork, Kenthapadi, McSherry, Mironov, and Naor]{Dwork2006OurData}
Cynthia Dwork, Krishnaram Kenthapadi, Frank McSherry, Ilya Mironov, and Moni Naor.
\newblock Our data, ourselves: Privacy via distributed noise generation.
\newblock In Serge Vaudenay, editor, \emph{Advances in Cryptology - EUROCRYPT 2006}, pages 486--503, Berlin, Heidelberg, 2006{\natexlab{a}}. Springer Berlin Heidelberg.
\newblock ISBN 978-3-540-34547-3.

\bibitem[Dwork et~al.(2006{\natexlab{b}})Dwork, McSherry, Nissim, and Smith]{Dwork2006Calibrating}
Cynthia Dwork, Frank McSherry, Kobbi Nissim, and Adam Smith.
\newblock Calibrating noise to sensitivity in private data analysis.
\newblock In Shai Halevi and Tal Rabin, editors, \emph{Theory of Cryptography}, pages 265--284, Berlin, Heidelberg, 2006{\natexlab{b}}. Springer Berlin Heidelberg.
\newblock ISBN 978-3-540-32732-5.
\newblock URL \url{https://doi.org/10.1007/11681878_14}.

\bibitem[Dwork et~al.(2014)Dwork, Talwar, Thakurta, and Zhang]{Dwork2014}
Cynthia Dwork, Kunal Talwar, Abhradeep Thakurta, and Li~Zhang.
\newblock Analyze gauss: Optimal bounds for privacy-preserving principal component analysis.
\newblock In \emph{Proceedings of the Forty-Sixth Annual ACM Symposium on Theory of Computing}, STOC '14, page 11–20, New York, NY, USA, 2014. Association for Computing Machinery.
\newblock ISBN 9781450327107.
\newblock URL \url{https://doi.org/10.1145/2591796.2591883}.

\bibitem[Gaboardi et~al.(2016)Gaboardi, Lim, Rogers, and Vadhan]{Gaboardi2016}
Marco Gaboardi, Hyun Lim, Ryan Rogers, and Salil Vadhan.
\newblock Differentially private chi-squared hypothesis testing: Goodness of fit and independence testing.
\newblock In Maria~Florina Balcan and Kilian~Q. Weinberger, editors, \emph{Proceedings of The 33rd International Conference on Machine Learning}, volume~48 of \emph{Proceedings of Machine Learning Research}, pages 2111--2120, New York, New York, USA, 20--22 Jun 2016. PMLR.
\newblock URL \url{https://proceedings.mlr.press/v48/rogers16.html}.

\bibitem[Gaboardi et~al.(2019)Gaboardi, Rogers, and Sheffet]{Gaboardi2019}
Marco Gaboardi, Ryan Rogers, and Or~Sheffet.
\newblock Locally private mean estimation: $z$-test and tight confidence intervals.
\newblock In Kamalika Chaudhuri and Masashi Sugiyama, editors, \emph{Proceedings of the Twenty-Second International Conference on Artificial Intelligence and Statistics}, volume~89 of \emph{Proceedings of Machine Learning Research}, pages 2545--2554. PMLR, 16--18 Apr 2019.
\newblock URL \url{https://proceedings.mlr.press/v89/gaboardi19a.html}.

\bibitem[Giddens and Liu(2023)]{DPpack}
Spencer Giddens and Fang Liu.
\newblock \emph{DPpack: Differentially Private Statistical Analysis and Machine Learning}, 2023.
\newblock URL \url{https://cran.r-project.org/package=DPpack}.
\newblock R package version 0.1.0.

\bibitem[Giddens et~al.(2023)Giddens, Zhou, Krull, Brinkman, Song, and Liu]{Giddens2023}
Spencer Giddens, Yiwang Zhou, Kevin~R. Krull, Tara~M. Brinkman, Peter X.~K. Song, and Fang Liu.
\newblock A differentially private weighted empirical risk minimization procedure and its application to outcome weighted learning.
\newblock \emph{ArXiv}, 2023.

\bibitem[Gillenwater et~al.(2021)Gillenwater, Joseph, and Kulesza]{Gillenwater2021}
Jennifer Gillenwater, Matthew Joseph, and Alex Kulesza.
\newblock Differentially private quantiles.
\newblock In Marina Meila and Tong Zhang, editors, \emph{Proceedings of the 38th International Conference on Machine Learning}, volume 139 of \emph{Proceedings of Machine Learning Research}, pages 3713--3722. PMLR, 18--24 Jul 2021.
\newblock URL \url{http://proceedings.mlr.press/v139/gillenwater21a/gillenwater21a.pdf}.

\bibitem[Google()]{GoogleDPpackage}
Google.
\newblock differential-privacy.
\newblock URL \url{https://github.com/google/differential-privacy}.

\bibitem[Guevara et~al.(2020)Guevara, Basaran, Kulankhina, and Ghazi]{GoogleDP}
Miguel Guevara, Mirac~Vuslat Basaran, Sasha Kulankhina, and Badih Ghazi.
\newblock Expanding our differential privacy library.
\newblock \url{https://opensource.googleblog.com/2020/06/expanding-our-differential-privacy.html}, June 2020.
\newblock Accessed: August 27, 2023.

\bibitem[Holohan et~al.(2019)Holohan, Braghin, Mac~Aonghusa, and Levacher]{diffprivlib}
Naoise Holohan, Stefano Braghin, P{\'o}l Mac~Aonghusa, and Killian Levacher.
\newblock Diffprivlib: the {IBM} differential privacy library.
\newblock \emph{ArXiv e-prints}, 1907.02444 [cs.CR], July 2019.

\bibitem[Karwa and Vadhan(2018)]{Karwa2018}
Vishesh Karwa and Salil Vadhan.
\newblock Finite sample differentially private confidence intervals.
\newblock In Anna~R. Karlin, editor, \emph{9th Innovations in Theoretical Computer Science Conference}, volume~94 of \emph{Leibniz International Proceedings in Informatics (LIPIcs)}, pages 44:1--44:9, Dagstuhl, Germany, 2018. Schloss Dagstuhl--Leibniz-Zentrum fuer Informatik.
\newblock ISBN 978-3-95977-060-6.
\newblock \doi{10.4230/LIPIcs.ITCS.2018.44}.

\bibitem[Kifer and Machanavajjhala(2011)]{Kifer2011}
Daniel Kifer and Ashwin Machanavajjhala.
\newblock No free lunch in data privacy.
\newblock In \emph{Proceedings of the 2011 ACM SIGMOD International Conference on Management of Data}, SIGMOD '11, page 193–204, New York, NY, USA, 2011. Association for Computing Machinery.
\newblock ISBN 9781450306614.
\newblock \doi{10.1145/1989323.1989345}.
\newblock URL \url{https://doi.org/10.1145/1989323.1989345}.

\bibitem[Kifer et~al.(2012)Kifer, Smith, and Thakurta]{Kifer2012}
Daniel Kifer, Adam Smith, and Abhradeep Thakurta.
\newblock Private convex empirical risk minimization and high-dimensional regression.
\newblock In Shie Mannor, Nathan Srebro, and Robert~C. Williamson, editors, \emph{Proceedings of the 25th Annual Conference on Learning Theory}, volume~23 of \emph{Proceedings of Machine Learning Research}, pages 25.1--25.40, Edinburgh, Scotland, 25--27 Jun 2012. PMLR.
\newblock URL \url{https://proceedings.mlr.press/v23/kifer12.html}.

\bibitem[Lada et~al.(2019)Lada, Peysakhovich, Aparicio, and Bailey]{Lada2019}
Akos Lada, Alexander Peysakhovich, Diego Aparicio, and Michael Bailey.
\newblock \emph{Observational Data for Heterogeneous Treatment Effects with Application to Recommender Systems}, page 199–213.
\newblock Association for Computing Machinery, 2019.
\newblock ISBN 9781450367929.
\newblock URL \url{https://doi.org/10.1145/3328526.3329558}.

\bibitem[Liu(2019{\natexlab{a}})]{Liu2019Generalized}
Fang Liu.
\newblock Generalized gaussian mechanism for differential privacy.
\newblock \emph{IEEE Transactions on Knowledge and Data Engineering}, 31\penalty0 (4):\penalty0 747--756, 2019{\natexlab{a}}.
\newblock URL \url{https://doi.org/10.1109/TKDE.2018.2845388}.

\bibitem[Liu(2019{\natexlab{b}})]{Liu2019Statistical}
Fang Liu.
\newblock Statistical properties of sanitized results from differentially private laplace mechanism with univariate bounding constraints.
\newblock \emph{Trans. Data Priv.}, 12:\penalty0 169--195, 2019{\natexlab{b}}.

\bibitem[Machanavajjhala et~al.(2008)Machanavajjhala, Kifer, Abowd, Gehrke, and Vilhuber]{Machanavajjhala2008}
Ashwin Machanavajjhala, Daniel Kifer, John Abowd, Johannes Gehrke, and Lars Vilhuber.
\newblock Privacy: Theory meets practice on the map.
\newblock In \emph{2008 IEEE 24th International Conference on Data Engineering}, pages 277--286, 2008.
\newblock \doi{10.1109/ICDE.2008.4497436}.

\bibitem[McSherry and Talwar(2007)]{McSherry2007}
Frank McSherry and Kunal Talwar.
\newblock Mechanism design via differential privacy.
\newblock In \emph{48th Annual IEEE Symposium on Foundations of Computer Science (FOCS'07)}, pages 94--103, 2007.
\newblock \doi{10.1109/FOCS.2007.66}.

\bibitem[McSherry(2009)]{mcsherry2009privacy}
Frank~D McSherry.
\newblock Privacy integrated queries: an extensible platform for privacy-preserving data analysis.
\newblock In \emph{Proceedings of the 2009 ACM SIGMOD International Conference on Management of data}, pages 19--30, 2009.

\bibitem[Narayanan and Shmatikov(2008)]{Narayanan2008}
Arvind Narayanan and Vitaly Shmatikov.
\newblock Robust de-anonymization of large sparse datasets.
\newblock In \emph{2008 IEEE Symposium on Security and Privacy (sp 2008)}, pages 111--125, 2008.
\newblock \doi{10.1109/SP.2008.33}.

\bibitem[Nayak(2020)]{MetaDP}
Chaya Nayak.
\newblock New privacy-protected facebook data for independent research on social media’s impact on democracy.
\newblock \url{https://research.facebook.com/blog/2020/02/new-privacy-protected-facebook-data-for-independent-research-on-social-medias-impact-on-democracy/}, February 2020.
\newblock Accessed: August 27, 2023.

\bibitem[{OpenMined}()]{OpenMinedDPpackage}
{OpenMined}.
\newblock {pyDP}.
\newblock URL \url{https://github.com/OpenMined/PyDP}.

\bibitem[Rahimi and Recht(2007)]{Rahimi2007}
Ali Rahimi and Benjamin Recht.
\newblock Random features for large-scale kernel machines.
\newblock In J~Platt, D~Koller, Y~Singer, and S~Roweis, editors, \emph{Advances in Neural Information Processing Systems}, volume~20. Curran Associates, Inc., 2007.
\newblock URL \url{https://proceedings.neurips.cc/paper/2007/file/013a006f03dbc5392effeb8f18fda755-Paper.pdf}.

\bibitem[Rahimi and Recht(2008)]{Rahimi2008}
Ali Rahimi and Benjamin Recht.
\newblock Weighted sums of random kitchen sinks: Replacing minimization with randomization in learning.
\newblock In D~Koller, D~Schuurmans, Y~Bengio, and L~Bottou, editors, \emph{Advances in Neural Information Processing Systems}, volume~21. Curran Associates, Inc., 2008.
\newblock URL \url{https://proceedings.neurips.cc/paper/2008/file/0efe32849d230d7f53049ddc4a4b0c60-Paper.pdf}.

\bibitem[Rubinstein and Ald{\`a}(2017)]{diffpriv}
Benjamin I.~P. Rubinstein and Francesco Ald{\`a}.
\newblock Pain-free random differential privacy with sensitivity sampling.
\newblock In Doina Precup and Yee~Whye Teh, editors, \emph{Proceedings of the 34th International Conference on Machine Learning}, volume~70 of \emph{Proceedings of Machine Learning Research}, pages 2950--2959. PMLR, 06--11 Aug 2017.
\newblock URL \url{https://proceedings.mlr.press/v70/rubinstein17a.html}.

\bibitem[Schnabel et~al.(2016)Schnabel, Swaminathan, Singh, Chandak, and Joachims]{Schnabel2016}
Tobias Schnabel, Adith Swaminathan, Ashudeep Singh, Navin Chandak, and Thorsten Joachims.
\newblock Recommendations as treatments: Debiasing learning and evaluation.
\newblock In \emph{Proceedings of the 33rd International Conference on International Conference on Machine Learning - Volume 48}, ICML'16, page 1670–1679. JMLR.org, 2016.

\bibitem[Shokri et~al.(2017)Shokri, Stronati, Song, and Shmatikov]{Shokri2017}
Reza Shokri, Marco Stronati, Congzheng Song, and Vitaly Shmatikov.
\newblock Membership inference attacks against machine learning models.
\newblock In \emph{2017 IEEE Symposium on Security and Privacy (SP)}, pages 3--18, 2017.
\newblock \doi{10.1109/SP.2017.41}.

\bibitem[Smith(2011)]{Smith2011}
Adam Smith.
\newblock Privacy-preserving statistical estimation with optimal convergence rates.
\newblock In \emph{Proceedings of the Forty-Third Annual ACM Symposium on Theory of Computing}, STOC '11, page 813–822, New York, NY, USA, 2011. Association for Computing Machinery.
\newblock ISBN 9781450306911.
\newblock \doi{10.1145/1993636.1993743}.
\newblock URL \url{https://doi.org/10.1145/1993636.1993743}.

\bibitem[Sun et~al.(2015)Sun, Wang, Yin, Yang, and Chang]{Sun2015}
Wei Sun, Pengyuan Wang, Dawei Yin, Jian Yang, and Yi~Chang.
\newblock Causal inference via sparse additive models with application to online advertising.
\newblock \emph{Proceedings of the AAAI Conference on Artificial Intelligence}, 29\penalty0 (1), Feb. 2015.
\newblock \doi{10.1609/aaai.v29i1.9156}.

\bibitem[Sweeney(2015)]{Sweeney2015}
Latanya Sweeney.
\newblock Only you, your doctor, and many others may know.
\newblock \emph{Technology Science}, September 2015.
\newblock URL \url{https://techscience.org/a/2015092903/}.

\bibitem[{TensorFlow}()]{TensorFlowPrivacy}
{TensorFlow}.
\newblock {TensorFlow Privacy}.
\newblock URL \url{https://github.com/tensorflow/privacy}.

\bibitem[{The OpenDP Team}()]{OpenDPpackage}
{The OpenDP Team}.
\newblock {OpenDP Library}.
\newblock URL \url{https://github.com/opendp/opendp}.

\bibitem[Torkzadehmahani et~al.(2019)Torkzadehmahani, Kairouz, and Paten]{torkzadehmahani2019dp}
Reihaneh Torkzadehmahani, Peter Kairouz, and Benedict Paten.
\newblock Dp-cgan: Differentially private synthetic data and label generation.
\newblock In \emph{Proceedings of the IEEE/CVF Conference on Computer Vision and Pattern Recognition Workshops}, pages 0--0, 2019.

\bibitem[Vaidya et~al.(2013)Vaidya, Shafiq, Basu, and Hong]{Vaidya2013}
Jaideep Vaidya, Basit Shafiq, Anirban Basu, and Yuan Hong.
\newblock Differentially private naive bayes classification.
\newblock In \emph{2013 IEEE/WIC/ACM International Joint Conferences on Web Intelligence (WI) and Intelligent Agent Technologies (IAT)}, volume~1, pages 571--576, 2013.
\newblock \doi{10.1109/WI-IAT.2013.80}.

\bibitem[Vinterbo(2018)]{PrivateLR}
Staal~A. Vinterbo.
\newblock \emph{PrivateLR: Differentially Private Regularized Logistic Regression}, 2018.
\newblock URL \url{https://CRAN.R-project.org/package=PrivateLR}.
\newblock R package version 1.2-22.

\bibitem[Wang et~al.(2015{\natexlab{a}})Wang, Sun, Yin, Yang, and Chang]{Wang2015Robust}
Pengyuan Wang, Wei Sun, Dawei Yin, Jian Yang, and Yi~Chang.
\newblock Robust tree-based causal inference for complex ad effectiveness analysis.
\newblock In \emph{Proceedings of the Eighth ACM International Conference on Web Search and Data Mining}, WSDM '15, page 67–76. Association for Computing Machinery, 2015{\natexlab{a}}.
\newblock ISBN 9781450333177.
\newblock \doi{10.1145/2684822.2685294}.

\bibitem[Wang et~al.(2015{\natexlab{b}})Wang, Fienberg, and Smola]{Wang2015}
Yu-Xiang Wang, Stephen Fienberg, and Alex Smola.
\newblock Privacy for free: Posterior sampling and stochastic gradient monte carlo.
\newblock In Francis Bach and David Blei, editors, \emph{Proceedings of the 32nd International Conference on Machine Learning}, volume~37 of \emph{Proceedings of Machine Learning Research}, pages 2493--2502, Lille, France, 07--09 Jul 2015{\natexlab{b}}. PMLR.
\newblock URL \url{https://proceedings.mlr.press/v37/wangg15.html}.

\bibitem[Wang et~al.(2019)Wang, Kifer, and Lee]{Wang2019}
Yue Wang, Daniel Kifer, and Jaewoo Lee.
\newblock Differentially private confidence intervals for empirical risk minimization.
\newblock \emph{Journal of Privacy and Confidentiality}, 9\penalty0 (1), Mar. 2019.
\newblock \doi{10.29012/jpc.660}.

\bibitem[Zhang et~al.(2017{\natexlab{a}})Zhang, Cormode, Procopiuc, Srivastava, and Xiao]{Zhang2017}
Jun Zhang, Graham Cormode, Cecilia~M. Procopiuc, Divesh Srivastava, and Xiaokui Xiao.
\newblock Privbayes: Private data release via bayesian networks.
\newblock \emph{ACM Transactions Database Systems}, 42\penalty0 (4), oct 2017{\natexlab{a}}.
\newblock ISSN 0362-5915.
\newblock \doi{10.1145/3134428}.

\bibitem[Zhang et~al.(2017{\natexlab{b}})Zhang, Cormode, Procopiuc, Srivastava, and Xiao]{zhang2017privbayes}
Jun Zhang, Graham Cormode, Cecilia~M Procopiuc, Divesh Srivastava, and Xiaokui Xiao.
\newblock Privbayes: Private data release via bayesian networks.
\newblock \emph{ACM Transactions on Database Systems (TODS)}, 42\penalty0 (4):\penalty0 1--41, 2017{\natexlab{b}}.

\bibitem[Zhang et~al.(2016)Zhang, Rubinstein, and Dimitrakakis]{Zhang2016}
Zuhe Zhang, Benjamin Rubinstein, and Christos Dimitrakakis.
\newblock On the differential privacy of bayesian inference.
\newblock \emph{Proceedings of the AAAI Conference on Artificial Intelligence}, 30\penalty0 (1), Mar. 2016.
\newblock URL \url{https://ojs.aaai.org/index.php/AAAI/article/view/10254}.

\bibitem[Zhao et~al.(2021)Zhao, Agrawal, Coburn, Asghar, Bhaskar, Kaafar, Webb, and Dickinson]{Zhao2021}
Benjamin Zi~Hao Zhao, Aviral Agrawal, Catisha Coburn, Hassan~Jameel Asghar, Raghav Bhaskar, Mohamed~Ali Kaafar, Darren Webb, and Peter Dickinson.
\newblock On the (in)feasibility of attribute inference attacks on machine learning models.
\newblock In \emph{2021 IEEE European Symposium on Security and Privacy}, pages 232--251, 2021.
\newblock \doi{10.1109/EuroSP51992.2021.00025}.

\bibitem[Zhao et~al.(2012)Zhao, Zeng, Rush, and Kosorok]{Zhao2012}
Yingqi Zhao, Donglin Zeng, A.~John Rush, and Michael~R. Kosorok.
\newblock Estimating individualized treatment rules using outcome weighted learning.
\newblock \emph{Journal of the American Statistical Association}, 107\penalty0 (499):\penalty0 1106--1118, 2012.
\newblock ISSN 01621459.
\newblock URL \url{http://www.jstor.org/stable/23427417}.

\end{thebibliography}

\end{document}